\documentclass{article}

\usepackage{microtype}
\usepackage{graphicx}
\usepackage{subcaption}
\usepackage{booktabs} 

\usepackage{hyperref}

\newcommand{\mintheta}{\thetab^{-}}

\newcommand{\minH}{\underline{\mathbf{H}}}

\usepackage[accepted]{icml2026}

\usepackage{amsmath}

    \newcommand\numberthis{\addtocounter{equation}{1}\tag{\theequation}}
\usepackage[capitalize,noabbrev]{cleveref}

\usepackage{amssymb}
\usepackage{mathtools}
\usepackage{amsthm} 
\usepackage{abbreviations}
\usepackage{natbib} 
\usepackage{thmtools}

\addtocontents{toc}{\protect\setcounter{tocdepth}{0}} 

\icmltitlerunning{Generalized Linear Bandits with Memory}

\begin{document}

\twocolumn[
  \icmltitle{Generalized Linear Bandits with Memory}

  \icmlsetsymbol{equal}{*}

  \begin{icmlauthorlist}
    \icmlauthor{Heesang Ann}{equal,snu}
    \icmlauthor{Hyunjun Choi}{equal,snu}
    \icmlauthor{Taehyun Hwang}{equal,snu}
    \icmlauthor{Younghoon Shin}{snu}
    \icmlauthor{Haeju Cheong}{ssg}
    \icmlauthor{Min-hwan Oh}{snu}
  \end{icmlauthorlist}

  \icmlaffiliation{snu}{Seoul National University, Seoul, Republic of Korea}
  \icmlaffiliation{ssg}{Shinsegae, Seoul, Republic of Korea}

  \icmlcorrespondingauthor{Min-hwan Oh}{minoh@snu.ac.kr}

  \icmlkeywords{Machine Learning, ICML}

  \vskip 0.3in
]

\printAffiliationsAndNotice{\icmlEqualContribution}

\begin{abstract}
We study generalized linear bandits with memory, an endogenous non-stationary setting in which rewards depend on past actions through a finite memory matrix. Building on prior work for linear models~\citep{clerici2024linear}, we show that the previously known $\tilde{\mathcal{O}}(T^{3/4})$ regret bound stems from a loose analysis, and we provide a sharpened analysis that recovers a $\tilde{\mathcal{O}}(\sqrt{T})$ regret rate in the linear case. We then extend this improvement to generalized linear models and propose a block-wise algorithm based on shrunken confidence bounds.
Our algorithm achieves a regret bound of $\tilde{\mathcal{O}}\left(\sqrt{mT} + d\sqrt{T} + \sqrt{\kappa}\, d^{2} m^{1/4} T^{1/4} + \kappa d^{2} \right)$, where $d$ denotes the feature dimension, $m$ the memory length, and $\kappa$ a curvature parameter of the link function. This attains a $\sqrt{T}$-type rate despite nonlinear rewards and memory effects. To the best of our knowledge, this analysis provides a unified treatment of memory-induced non-stationarity and nonlinear link functions, while ensuring that the leading regret term is independent of the curvature of the link function. We conduct numerical experiments that are consistent with our theoretical findings.
\end{abstract}

\section{Introduction}
In real-world recommendation systems, users' preferences are often inherently nonstationary and may evolve through repeated interactions.
For example, repeatedly recommending similar items may induce fatigue, making even initially preferred items less effective over time.
Such action-induced nonstationarity has been studied through rotting bandits~\citep{Bouneffouf2016multi-armed, levine2017rotting, seznec2019rotting}, rising bandits~\citep{metelli2022stochastic}, and related models where past actions affect future rewards~\citep{kleinberg2018recharging, basu2019blocking, cella2020stochastic, leqi2021rebounding, simchi2021dynamic}.
However, these models are mostly developed for finite-arm settings, and the reward evolution is typically tied to the selected arm itself.
They therefore do not fully capture feature-level cross-effects, where one action can alter the rewards of similar or related future actions.

A recent work on linear bandits with memory~\citep{clerici2024linear} provides a principled framework for modeling such endogenous non-stationarity.
In this model, the expected reward is linear in the current action, while the effective preference parameter is transformed by a memory matrix determined by previously selected actions (Figure~\ref{fig:BM}).
This formulation captures both rotting and rising effects over structured action spaces, allowing past actions to influence the rewards of related future actions.
To handle the resulting long-term dependence, \citet{clerici2024linear} reduce the problem to that of learning a cyclic block policy: the learner selects a block of actions, executes it, and updates its estimate only after observing the rewards from the block.
This reduction reveals a key structural difficulty of bandits with memory: the learner must commit to multiple actions before receiving feedback and updating the model.

The best known regret guarantee for this linear bandit with memory model is of order \(\widetilde{\Ocal}(T^{3/4})\)~\citep{clerici2024linear}.
This rate is obtained by balancing two sources of error.
The first is the approximation error incurred by restricting the policy class to cyclic block policies, which decreases as the block length grows.
The second is the estimation error for learning a good block, which was previously bounded by a term that scales linearly with the block length.
This raises our first research question:
    \textit{Is the \(\widetilde{\Ocal}(T^{3/4})\) regret barrier intrinsic to linear bandits with memory, or an artifact of loose block-wise uncertainty control?}

Our observation is that this suboptimal regret rate is not intrinsic to memory-induced non-stationarity.
Rather, it stems from a mismatch in the existing analysis.
The estimator is updated only after a block is executed, but the previous analysis controls within-block uncertainty using an auxiliary Gram matrix that grows inside the block.
This treats the actions in a block as sequentially adaptive decisions, even though they are chosen jointly from the same pre-block information.
Since the learner commits to the entire block before receiving feedback, a block is better viewed as a single combinatorial decision whose components are the actions placed in the block.
This block-level viewpoint allows us to control within-block uncertainty using the information structure available before execution and removes the loose linear dependence on the block length.

Answering the first question is not only important for sharpening the regret guarantee of the linear model.
It also identifies the main principle needed for learning with memory: one should exploit the joint structure of actions within a block while respecting the fact that feedback and parameter updates are delayed until the block is completed.
This principle becomes substantially more delicate once the reward model is no longer linear.
Indeed, linear rewards are often too restrictive for applications such as clicks, conversions, or ratings, where the expected reward is bounded and nonlinear in a latent preference score~\citep{filippi2010parametric, li2017provably, faury2020improved}.
\textit{Generalized linear models} (GLMs) provide a natural extension, allowing the reward mean to be represented through a link function.
However, they also change the nature of uncertainty control.

In linear models, confidence widths depend only on feature geometry.
In GLMs~\citep{faury2020improved, zhang2025generalized}, they also depend on the local slope, or curvature, of the link function.
This dependence is commonly captured through a curvature parameter, which may significantly affect the regret bound.
In the memory setting, this issue is further complicated by block-wise decision making.
Since the learner constructs an entire block before observing feedback, it cannot update either the parameter estimate or the curvature-weighted Gram matrix while choosing actions within the block.
Thus, the same limited-adaptivity structure that causes the loose analysis in the linear case now interacts with curvature-dependent confidence geometry.
This motivates our second research question:
    \textit{Can block-wise uncertainty control be extended to GLMs when curvature information is unavailable within a block?}

In this work, we answer these questions in two steps. 
First, for linear bandits with memory, we revisit the regret analysis of the \texttt{OFUL-memory} algorithm~\citep{clerici2024linear} and reinterpret each block as a single combinatorial decision.
This removes the artificial within-block design mismatch in the previous analysis and improves the regret guarantee from \(\widetilde \Ocal(T^{3/4})\) to \(\widetilde \Ocal(\sqrt T)\), without modifying the original algorithm.
Second, we extend this block-level principle beyond linear rewards by introducing generalized linear bandits with memory and developing a block-wise algorithm based on \textit{online mirror descent} (OMD) estimation and a block-adaptive confidence bonus.
Rather than assigning each action in a block a confidence width computed only from pre-block data, our algorithm refines the uncertainty calculation as the block is constructed.
In particular, it incorporates the feature information of actions already placed in the current block, while relying on auxiliary estimators to handle the curvature-dependent confidence geometry of GLMs without observing within-block feedback.
This design yields a \(\widetilde{\Ocal}(\sqrt T)\)-type regret bound whose leading term does not depend on the curvature parameter of the link function.
Our main contributions are summarized as follows.

\begin{figure*}[t]
    \centering
    \includegraphics[width=0.95\textwidth]{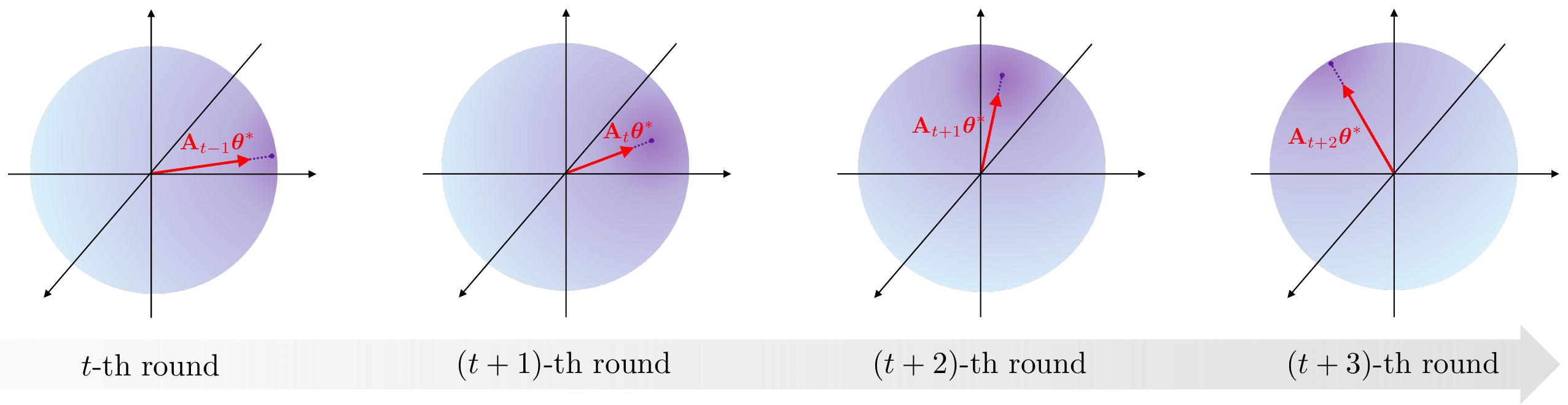}
    \caption{
    Illustration of how nonstationary preferences are modeled in the bandit with memory setting. The sphere represents the action space, and the red arrow denotes the preference vector $\Ab\thetab^*$. The preference vector is transformed by a memory matrix that depends on previously selected arm features, resulting in different optimal arms at each round. The color gradient visualizes the distribution of expected rewards over actions at each round: actions in darker purple regions yield higher rewards, whereas those in lighter blue regions yield lower rewards.}
    \label{fig:BM}
\end{figure*}

\begin{itemize}
    \item 
    We provide a sharper regret analysis of the \texttt{OFUL-memory} algorithm for linear bandits with memory~\citep{clerici2024linear}.
    Without modifying the algorithm, our analysis improves its regret guarantee from
    \(\widetilde O(\sqrt{dm}\,T^{3/4})\) to
    \(\widetilde O(\sqrt{dmT}+d\sqrt T)\),
    where \(T\) is the horizon, \(d\) is the feature dimension, and \(m\) is the memory length.
    The key observation is that the previous analysis treats actions within a block as if they were sequentially controlled, leading to an estimation error that scales linearly with the block length.
    Instead, by viewing each block as a single combinatorial decision and adapting techniques from combinatorial bandits~\citep{jin2021shrinking}, we obtain tighter block-wise uncertainty bounds.
    
    \item 
    We introduce generalized linear bandits with memory, extending the linear memory model to nonlinear reward distributions through a link function.
    This framework captures both memory-induced non-stationarity and nonlinear reward feedback, such as binary clicks or bounded ratings.

    \item 
    We propose a block-wise confidence-bound algorithm, \(\algname\), for generalized linear bandits with memory.
    The proposed algorithm estimates the unknown parameter using OMD and uses auxiliary estimators to approximate the curvature-dependent confidence geometry during block construction.
    Based on this estimated geometry, it shrinks the confidence bonus by incorporating the feature information of actions already placed in the current block, thereby exploiting within-block feature information without requiring within-block feedback or parameter updates.
    We prove that the proposed algorithm achieves a \(\widetilde{\Ocal}(\sqrt T)\) regret bound whose leading term is independent of the curvature parameter of the link function.
    To the best of our knowledge, this is the first such guarantee for action-induced non-stationary generalized linear bandits.

    \item 
    We provide numerical experiments that support our theoretical findings and illustrate the empirical performance of the proposed methods.
\end{itemize}

\section{Preliminaries}

\paragraph{Notations.} 
For a vector \(\xb \in \RR^d\), we write \(\|\xb\|_2\) for the Euclidean norm and
\(\|\xb\|_{\Ab} = \sqrt{\xb^\top \Ab \xb}\) for the weighted norm induced by a positive semi-definite matrix \(\Ab \in \RR^{d \times d}\). We denote a $d$-dimensional ball with radius $r$ by $\Bcal_d(r)$. 
For a symmetric matrix \(\Ab\), we denote its spectral norm by \(\|\Ab\|_*\), i.e.,
$
    \|\Ab\|_* = \max_i |\lambda_i(\Ab)|,
$
where \(\lambda_i(\Ab)\) are the eigenvalues of \(\Ab\).
For a function \(\mu\), we denote its first and second derivatives by
\(\dot{\mu}\) and \(\ddot{\mu}\), respectively.

\subsection{Generalized Linear Bandits with Memory}
In this section, we introduce the model of generalized linear bandits with memory, which extends the linear bandit with memory framework of~\citet{clerici2024linear} to generalized linear reward models.
At each round $t \in [T]$, the agent selects an action $\xb_t$ from a (possibly infinite) action set $\Xcal \subset \RR^d$, and observes a stochastic reward $y_t$.
We consider a generalized linear model with a canonical exponential family reward distribution~\citep{mccullagh1989generalized}, where the conditional mean is parameterized by an \textit{unknown} parameter $\thetab^* \in \mathbb{R}^d$ through a known, strictly increasing link function $\mu(\cdot)$.
In contrast to stationary generalized linear bandits~\citep{li2017provably}, the expected reward at time $t$ is influenced not only by the current action but also by previously selected actions, capturing non-stationary effects induced by memory.
Formally, this dependence is modeled through a memory matrix $\Ab_{t-1} := \Ab(\xb_{t-m}, \ldots, \xb_{t-1})$ where $m \ge 0$ denotes the memory length.
For notational convenience, we set \(\xb_j=\mathbf 0\) for all \(j\le 0\).
Then, the conditional expectation of the reward satisfies
\begin{equation} \label{eq:glm-memory-mean}
    \EE\left[y_t \mid \Fcal_{t-1}, \xb_t \right]
    =
    \mu\left( \langle \xb_t, \Ab_{t-1} \thetab^* \rangle \right),
\end{equation}
where $\Fcal_{t-1}$ is the filtration generated by the sequence of past actions and rewards up to round $t-1$.
We define the noise term as
\(
\eta_t := y_t - \mu(\langle \xb_t, \Ab_{t-1}\thetab^* \rangle),
\)
which satisfies
\(
\mathbb E[\eta_t \mid \Fcal_{t-1}] = 0.
\)
Throughout the main text, we work with the conditional mean model in Eq.~\eqref{eq:glm-memory-mean}; a canonical exponential-family specification that induces this model is provided in Appendix~\ref{appx:canonical-family}.

To be specific, we define the memory matrix $\Ab(\cdot)$ as
\begin{equation*}
    \Ab(\xb_1, \ldots, \xb_m):= \left( \Ib_d + \sum_{s=1}^m \xb_{s} \xb_{s}^\top \right)^\gamma \, ,
\end{equation*}
where $\gamma \in \RR$ controls both the strength and the nature of the memory effect. 
In particular, $\gamma > 0$ corresponds to rising (or excitation) effects, when $\gamma < 0$ captures rotting (or satiation) effects, and $\gamma=0$ recovers the stationary generalized linear bandit model.
The memory matrix $\Ab_{t-1}$ aggregates information from previously selected actions and naturally encodes directions along which excitation or satiation occurs via its eigenstructure. 

\begin{remark}
The above definition is adopted without loss of generality. More generally, one can define and analyze a model with the memory matrix $\left( \Ab_0 + \sum_{s=1}^m \xb_{s} \xb_{s}^\top \right)^\gamma$, provided that $\|(\Ab_0)^\gamma\|_* \le C$ for some $C \in \mathbb{R}$.
\end{remark}

In this setting, the goal of the agent is to minimize the expected cumulative regret over \(T\) rounds, defined as
\begin{equation*} 
    \regret_T
    =
    \mathrm{OPT}(T)
    -
    \EE\left[ \sum_{t=1}^T y_t \right] \, ,
\end{equation*}
where \(\mathrm{OPT}(T)\) denotes the maximum expected cumulative reward achieved by an optimal sequence of actions:
\begin{equation*}
    \mathrm{OPT}(T)
    :=
    \max_{\xb_1,\ldots,\xb_T \in \Xcal}
    \sum_{t=1}^T
    \mu\!\left(
        \left\langle
            \xb_t,
            \Ab(\xb_{t-m},\ldots,\xb_{t-1})\thetab^*
        \right\rangle
    \right) \, .
\end{equation*}

Following standard assumptions in the generalized linear bandit literature~\citep{li2017provably, russac2020algorithms, zhang2023online, wang2023revisiting, zhang2025generalized}, we make the following assumptions and introduce the curvature-related parameter of the link function.

\begin{assumption}[Boundedness] \label{assm:bounded domain}
    The action set \(\Xcal \subset \RR^d\) is closed and satisfies
    \(\|\xb\|_2 \le 1\) for all \(\xb \in \Xcal\).
    There exists \(C_\theta > 0\) such that \(\|\thetab^*\|_2 \le C_\theta\).
\end{assumption}

Under Assumption~\ref{assm:bounded domain}, the memory matrix is uniformly bounded.
Indeed, for any history \((\xb_{t-m},\ldots,\xb_{t-1})\), we have
$
    \left\|
    \Ib + \sum_{s=1}^m \xb_{t-s}\xb_{t-s}^{\top}
    \right\|_*
    \le 1 + \sum_{s=1}^m \|\xb_{t-s}\|_2^2
    \le m+1 \, .
$
Hence, with \(\gamma^+ := \max\{\gamma,0\}\), 
$
    \|\Ab_{t-1}\|_*
    \le (m+1)^{\gamma_+} \, .
$
We denote this quantity by
$
    R := (m+1)^{\gamma_+} \, .
$
Consequently, the effective preference parameter satisfies
\[
    \|\Ab_{t-1}\thetab^*\|_2
    \le
    \|\Ab_{t-1}\|_* \|\thetab^*\|_2
    \le R C_\theta .
\]
In particular, in the rotting case \(\gamma < 0\), we have \(R=1\).

\begin{assumption}[Link function] \label{assm:bounded link}
    The link function $\mu$ is twice differentiable over its feasible domain. 
    There exist constants $l_\mu > 0$ and $U_\mu > 0$ such that $l_\mu \le \dot \mu(z) \le U_\mu$ for all $z \in [-RC_\theta, RC_\theta]$.
\end{assumption}

\begin{assumption}[Self-concordance] \label{assm:self concordance}
The link function satisfies $\lvert \ddot \mu(z)\rvert \le \nu \cdot \dot \mu(z)$ for all $z\in \RR$.
\end{assumption}

\begin{definition}[Curvature-related nonlinearity parameter]
    We quantify the degree of nonlinearity of the link function over the feasible
    decision set by
    $
        \kappa
        :=
        \sup_{\xb \in \Xcal,\, \thetab \in \Dcal}
        1/\dot{\mu}(\xb^\top \thetab) ,
    $
    where
    $
        \Dcal
        :=
        \left\{
        \thetab \in \RR^d : \|\thetab \|_2 \le R C_\theta
        \right\} \, .
    $
\end{definition}

\subsection{Approximating Cyclic Policy}

A natural but naive approach is to select actions greedily with respect to the current expected reward.
However, because rewards depend on past actions, a currently optimal action may deteriorate the future reward structure through rotting effects or excessively reinforce certain directions through rising effects, potentially leading to substantial long-term losses.
For the linear reward model, \citet{clerici2024linear} show that even an oracle greedy policy, which has full knowledge of \(\thetab^*\), can suffer linear regret in the worst case.
This obstruction is not specific to linear rewards.
In generalized linear models, the link function is strictly increasing, so myopically maximizing
\(\mu(\langle \xb,\Ab_{t-1}\thetab^*\rangle)\)
is equivalent to myopically maximizing the latent score
\(\langle \xb,\Ab_{t-1}\thetab^*\rangle\).
Moreover, under our boundedness and link-function assumptions, the derivative
of \(\mu\) is bounded away from zero on the feasible domain.
Hence, a constant gap in latent scores translates into a constant gap in
expected rewards.
The following proposition formalizes that oracle greedy can suffer linear regret even in generalized linear bandits with memory.

\begin{restatable}[Failure of oracle greedy]{proposition}{FailureOfOracleGreedy}
    \label{prop:failure-of-oracle-greedy}
    Let \(\mu\) be a continuous and strictly increasing link function on the feasible score domain.
    Then, for generalized linear bandits with memory, there exist both rotting and rising instances in which the oracle greedy policy
    $
        \pi_t^{\mathrm{greedy}}
        \in
        \argmax_{\xb \in \Xcal}
        \mu\!\left(\langle \xb, \Ab_{t-1}\thetab^* \rangle\right)
    $
    suffers linear regret.    
\end{restatable}

To address this challenge, \citet{clerici2024linear} restrict the policy class to cyclic policies, which repeatedly execute a fixed action block of length \(m+L\).
For any \(m,L\ge1\) and any action block \(\xb=(\xb_1,\ldots,\xb_{m+L})\), define its proxy reward as
\begin{equation}
\label{eq:proxy reward}
    \tilde r(\xb)
    :=
    \sum_{t=m+1}^{m+L}
    \mu\!\left(
        \left\langle
            \xb_t,
            \Ab(\xb_{t-m},\ldots,\xb_{t-1})\thetab^*
        \right\rangle
    \right).
\end{equation}
They show that the approximation error incurred by repeatedly playing an optimal proxy block
$
    \tilde{\xb}
    \in
    \argmax_{\xb\in\Xcal^{m+L}} \tilde r(\xb)
$
in a cyclic manner is of order
$
    \Theta\!\left(\frac{mT}{m+L}\right)
$
up to lower-order boundary terms (Appendix~\ref{appx:approximation}).

With this cyclic approximation, the memory dependence is localized within each block, and the problem of identifying an optimal block can be reduced to a stationary linear bandit problem.
The regret analysis then balances two terms: the approximation error induced by restricting to cyclic policies and the estimation error incurred in learning the optimal block.
Optimizing the block length under this trade-off yields the \(\widetilde{\Ocal}(T^{3/4})\) regret bound of~\citet{clerici2024linear}.

In the following section, we build on this block-wise reduction and study how to learn the optimal block more efficiently.
For simplicity, we assume throughout that \(T\) is a multiple of \(m+L\); this only affects lower-order boundary terms and does not change the regret order.
We show that the same linear-bandit algorithm admits a sharper \(\widetilde{\Ocal}(\sqrt T)\) regret bound under a refined analysis, and then extend the block-wise learning principle to generalized linear rewards.

\section{Main Results}

\subsection{Improved Regret Bound for Linear Rewards}

We first revisit the linear bandit with memory model, which corresponds to the identity-link special case of our generalized linear reward model.
Our goal in this section is to show that the previously known \(\widetilde{\Ocal}(T^{3/4})\) regret bound for \texttt{OFUL-memory}~\citep{clerici2024linear} is not intrinsic to the memory structure.
Rather, the suboptimal rate comes from a loose treatment of within-block uncertainty.
We show that, without modifying the algorithm, a refined block-level analysis yields a \(\widetilde{\Ocal}(\sqrt T)\)-type regret bound.

\paragraph{Existing analysis by~\citet{clerici2024linear}.}
\citet{clerici2024linear} reduce the problem of identifying an optimal block maximizing the proxy reward in Eq.~\eqref{eq:proxy reward} to a stationary linear bandit problem, and propose a block variant of OFUL, referred to as \texttt{OFUL-memory} (Algorithm~\ref{alg:OM}).
Their analysis obtains an estimation regret bound of order \(\widetilde{\Ocal}(Ld\sqrt T)\) with respect to the optimal block.
Balancing this estimation error with cyclic approximation error of order \(\Ocal(T/L)\) leads to the choice \(L=\Ocal(T^{1/4})\), and to the regret bound \(\widetilde{\Ocal}(T^{3/4})\).

The key source of looseness is how the within-block uncertainty is aligned with the confidence radius.
In this block-wise algorithm, the reward parameter is estimated only after each block is completed, so the confidence set is constructed using a block-level Gram matrix.
However, to apply an elliptical-potential argument to individual actions inside a block, the previous analysis introduces an auxiliary Gram matrix that is incrementally updated within the block.

To make this mismatch explicit, let \(\bb_{\tau,i}:=\Ab_{\tau,i-1}\xb_{\tau,i}\) be the memory-dependent feature of the \(i\)-th action in block \(\tau\).
Let \(\Vb_\tau\) denote the Gram matrix formed from past blocks, and let \(\Vb_{\tau,i}:=\Vb_\tau+\sum_{j=m+1}^{i}\bb_{\tau,j}\bb_{\tau,j}^{\top}\) be the auxiliary Gram matrix that also includes actions already placed in the current block.
Thus, \(\Vb_\tau\) is used to construct the confidence set at the beginning of the block, whereas \(\Vb_{\tau,i}\) is used to sum within-block uncertainties.
The previous analysis controls the prediction error as
\[
    \langle \bb_{\tau,i}, \widetilde{\thetab}_{\tau}-\thetab^* \rangle
    \le
    \|\widetilde{\thetab}_{\tau}-\thetab^*\|_{\Vb_{\tau,i-1}}
    \|\bb_{\tau,i}\|_{\Vb_{\tau,i-1}^{-1}},
\]
where \(\widetilde{\thetab}_{\tau}\) is the optimistic parameter selected at the beginning of block \(\tau\).
The issue is that \(\widetilde{\thetab}_{\tau}\) is certified only in the block-level norm induced by \(\Vb_\tau\), whereas the bound above requires certification in the within-block norm induced by \(\Vb_{\tau,i-1}\).
Since \(\Vb_{\tau,i-1}\) contains actions from the current block, the confidence radius must be enlarged, and this enlargement scales with the block length \(L\).
Consequently, the estimation regret is bounded as
$
    \widetilde{\Ocal}(\beta L\sqrt{dT}),
$
so that \(L\) appears multiplicatively in the leading term.

\paragraph{Block-level combinatorial viewpoint.}
Our analysis is based on the observation that actions within a block are chosen jointly before any feedback from that block is observed.
Thus, a block is more naturally viewed as a single combinatorial decision~\citep{qin2014contextual, wen2015efficient, jin2021shrinking} rather than as a sequence of adaptively selected actions.
This viewpoint suggests that one should not first move to the within-block Gram matrix and then enlarge the confidence radius.
Instead, we keep the confidence radius with respect to the block-level Gram matrix \(\Vb_\tau\), which is the matrix used to construct the confidence set at the beginning of block \(\tau\), and handle the remaining mismatch separately.

Concretely, for the memory-dependent feature \(\bb_{\tau,i}\), our analysis controls the prediction error as
\begin{align*}
    & \ips{\bb_{\tau,i}}{\tilde{\thetab}_{\tau}-\thetab^*}
     \le
    \|
        \widetilde{\thetab}_{\tau}-\thetab^*
    \|_{\Vb_\tau}
    \left\|
        \bb_{\tau,i}
    \right\|_{\Vb_\tau^{-1}}
    \\
    &
    \le
    \beta
    \left\|
        \bb_{\tau,i}
    \right\|_{\Vb_{\tau,i-1}^{-1}}
    +
    \beta
    (
        \left\|
            \bb_{\tau,i}
        \right\|_{\Vb_\tau^{-1}}
        -
        \left\|
            \bb_{\tau,i}
        \right\|_{\Vb_{\tau,i-1}^{-1}}
    ) \, .
\end{align*}

The first term is controlled by the standard elliptical-potential argument and contributes
\(\widetilde{\Ocal}(\beta\sqrt{dT})\).
The second term captures the discrepancy between the block-level and within-block Gram matrices; using an eigenvalue telescoping argument from~\citet{jin2021shrinking}, it can be bounded by \(\widetilde{\Ocal}(\beta Ld)\).
Consequently, the block length \(L\) no longer multiplies the \(\sqrt T\) term.
This is the key sense in which the combinatorial viewpoint decouples the leading estimation term from the block length, yielding a sharper regret bound without modifying the original algorithm.

\begin{theorem}[Improved Regret of~\texttt{OFUL-memory}~\citep{clerici2024linear}, informal] \label{thm:improved regret for linear model informal}
    Suppose that Assumption~\ref{assm:bounded domain} holds. By setting the block length $L = \lfloor m^{1/2} d^{-1/2}T^{1/2} -m \rfloor$, 
    \texttt{OFUL-memory} (Algorithm~\ref{alg:OM}) achieves a cumulative regret bound of order
    \begin{equation*}
        \regret_T = \widetilde{\Ocal}\left(R \sqrt{m d T} + \sqrt{d \max\{d, R^2\}T} \right) \, .
    \end{equation*}
\end{theorem}

\paragraph{Discussion of Theorem~\ref{thm:improved regret for linear model informal}.}    

Theorem~\ref{thm:improved regret for linear model informal} shows that a \(\widetilde{\Ocal}(\sqrt T)\) regret bound is achievable even under memory-induced non-stationarity, improving the previous \(\widetilde{\Ocal}(T^{3/4})\) guarantee of~\citet{clerici2024linear}.
The key improvement is to decouple the leading \(\sqrt T\)-scaling estimation term from the block length \(L\).
To see why this is essential, recall that the regret under a cyclic block policy
can be decomposed into an approximation error and an estimation error.
The approximation error comes from restricting the policy to cyclic blocks and
scales as \(\Ocal(\frac{mT}{m+L})\), which decreases as \(L\) grows.
On the other hand, the estimation error depends on how costly it is to learn an
optimal block.
If the estimation error can be written as \(L^\alpha\sqrt T+h(L,T)\), where \(h(L,T)\) collects the remaining lower-order terms such as logarithmic factors and residual block-length-dependent terms,
then balancing the two errors gives
\[
    \regret_T
    \lesssim
    \frac{m}{m+L}T
    +
    L^\alpha\sqrt T
    +
    h(L,T),
\]
and choosing \(L\simeq T^{1/(2+2\alpha)}\) yields a regret rate of order
$
    \Ocal\!\left(
        T^{\frac{1+2\alpha}{2+2\alpha}}
        + h(L,T)
    \right).
$
Thus, recovering a \(\sqrt T\)-type rate requires \(\alpha=0\), namely that the
block length does not multiply the leading \(\sqrt T\) estimation term.

The previous analysis effectively has \(\alpha=1\), because the within-block
uncertainty control yields an estimation term proportional to \(L\sqrt T\).
Our block-level combinatorial analysis removes this multiplicative dependence:
the leading estimation term scales as \(\sqrt T\), while the residual
block-length-dependent term is handled separately through an eigenvalue
telescoping argument.
This is precisely why the refined analysis leads to the sharper regret bound in
Theorem~\ref{thm:improved regret for linear model informal}.
Complete proofs are provided in Appendix~\ref{appx:proof for linear case}.

\subsection{Algorithm for Generalized Linear Models}

\subsubsection{Algorithm : $\algname$}
\begin{algorithm}[t!]
    \caption{$\algname$ (\textbf{G}eneralized \textbf{L}inear \textbf{B}andits with \textbf{M}emory-\textbf{S}hrunken \textbf{C}onfidence \textbf{B}ounds)}
    \label{alg:main algorithm}
    \begin{algorithmic}[1]
        \STATE \textbf{Inputs: } regularization parameter $\lambda$, number of rounds $T$, block length $L$, step size $\eta$, confidence radius $\beta$
        \STATE Initialize \(\thetab_1 =  \zero_d\), \(\Hb_1=\Vb_1 = \lambda \Ib_d\), and
        \(\tau_{\mathrm{warm}} = 0\)
        \STATE Set the warm-up block $\xb^{\mathrm{warm}}$ as a fixed conservative E-optimal design block 
        \FOR{\(\tau=1,2,\ldots,\lfloor T/(m+L)\rfloor\)}
            \IF[\hfill $\triangleright$ \textit{information warm-up}]{\(\lambda_{\min}(\Vb_\tau)<L/d\)}
                \STATE Set \(\xb_\tau\leftarrow \xb^{\mathrm{warm}}\) and \(\tau_{\mathrm{warm}}\leftarrow \tau_{\mathrm{warm}}+1\) 
                \STATE Update  $\Vb_{\tau+1} = \Vb_\tau + \sum_{i=m+1}^{m+L} \bb_{\tau,i} \bb_{\tau,i}^\top$
            \ELSE[\hfill 
            $\triangleright$ 
            \textit{block-safe SCB}]
                \STATE 
                Select
                \begin{equation*}
                    \xb_\tau \in \argmax_{\xb \in \Xcal^{m+L}}
                    \hat r_\tau (\xb)
                    + \max\{ \Gamma^{\mathrm{SCB}}_{\tau}(\xb), \Gamma^{\mathrm{safe}}_\tau (\xb) \}    
                \end{equation*}
            \ENDIF
            \STATE Play \(\xb_\tau\) and observe rewards
            \STATE Set \(\thetab_{\tau,m+1}\gets\thetab_\tau\) and \(\Hb_{\tau,m+1}\gets\Hb_\tau\)
            \FOR{\(i=m+1,\ldots,m+L\)}
                \STATE 
                Compute
                \(\thetab_{\tau,i+1}\) via the OMD update in Eq.~\eqref{eq:omd_update}
                \STATE Update \(\Hb_{\tau,i+1}\gets\Hb_{\tau,i}+\dot\mu(\bb_{\tau,i}^\top\thetab_{\tau,i+1})\bb_{\tau,i}\bb_{\tau,i}^\top\)
            \ENDFOR
            \STATE Set \(\thetab_{\tau+1}\gets\thetab_{\tau,m+L+1}\) and \(\Hb_{\tau+1}\gets\Hb_{\tau,m+L+1}\)
        \ENDFOR
    \end{algorithmic}
\end{algorithm}

In this section, we introduce \(\algname\), a block-wise confidence-bound algorithm for generalized linear bandits with memory. Recall that, under the cyclic block reduction, the learning problem is to identify a block that maximizes the proxy reward in Eq.~\eqref{eq:proxy reward}. This requires selecting a block of \(m+L\) actions while estimating the unknown reward parameter \(\thetab^*\) from nonlinear feedback collected in previously executed blocks. 
A straightforward block-wise optimistic strategy would control each action in a candidate block using only the information available before the block is executed. Such a strategy is statistically valid, but it does not exploit the feature information of actions already placed in the same block. Consequently, its leading regret term couples the memory length \(m\) and the feature dimension \(d\), as in the linear case of Theorem~\ref{thm:improved regret for linear model informal}.
Our goal is to use the geometry of the partially constructed block to shrink the confidence bonus, while preserving a valid block-level comparison with the optimal proxy block.
The proposed algorithm has three components.
First, it estimates the unknown parameter using OMD.
Second, it performs an \textit{information warm-up phase} until the curvature-weighted Gram matrix is sufficiently well-conditioned.
Third, after warm-up, it selects a block using a \textit{block-safe shrunken confidence bound}. The confidence bound combines an action-wise shrunken bonus, which exploits feature information from actions already placed in the block, with a block-level safety bonus, which guarantees optimism for the entire block.

\paragraph{Parameter estimation.} 
For parameter estimation, we build on recent advances in generalized linear bandits~\citep{zhang2023online, zhang2025generalized} and use an OMD update. Let \(\xb_\tau=(\xb_{\tau,1},\ldots,\xb_{\tau,m+L})\) be the block selected at block time \(\tau\), and let \((y_{\tau,1},\ldots,y_{\tau,m+L})\) be the corresponding rewards. For each \(i\in[m+L]\), let \(\Ab_{\tau,i-1}\) denote the memory matrix generated by the last \(m\) actions before \(\xb_{\tau,i}\), and define the memory-dependent feature $\bb_{\tau,i}:=\Ab_{\tau,i-1}^{\top}\xb_{\tau,i}$.
Then \(\langle \xb_{\tau,i},\Ab_{\tau,i-1}\thetab\rangle=\bb_{\tau,i}^{\top}\thetab\). For \(i=m+1,\ldots,m+L\), the negative log-likelihood loss induced by the canonical exponential-family model is \[ \ell_{\tau,i}(\thetab) \propto \frac{ \psi(\bb_{\tau,i}^{\top}\thetab) - y_{\tau,i}\bb_{\tau,i}^{\top}\thetab }{g(\phi)}, \] where \(\psi\) is the log-partition function and terms independent of \(\thetab\) are omitted. At the beginning of block \(\tau\), set \(\Hb_{\tau,m+1}=\widetilde{\Hb}_{\tau,m+1}=\Hb_\tau\) and \(\thetab_{\tau,m+1}=\thetab_\tau\). For each \(i=m+1,\ldots,m+L\), define the linearized loss \(\widetilde\ell_{\tau,i}(\thetab) := \langle\nabla\ell_{\tau,i}(\thetab_{\tau,i}),\thetab-\thetab_{\tau,i}\rangle\), and set \(\widetilde{\Hb}_{\tau,i}:=\Hb_{\tau,i}+\nabla^2\ell_{\tau,i}(\thetab_{\tau,i})\). The OMD update is \begin{equation} \label{eq:omd_update} \thetab_{\tau,i+1} \in \argmin_{\thetab\in\Bcal_d(C_\theta)} \left\{ \widetilde\ell_{\tau,i}(\thetab) + \frac{1}{2\eta} \|\thetab-\thetab_{\tau,i}\|_{\widetilde{\Hb}_{\tau,i}}^2 \right\}. \end{equation} After the update, set \(\Hb_{\tau,i+1}:=\Hb_{\tau,i}+\nabla^2\ell_{\tau,i}(\thetab_{\tau,i+1})\). At the end of the block, set \(\thetab_{\tau+1}:=\thetab_{\tau,m+L+1}\) and \(\Hb_{\tau+1}:=\Hb_{\tau,m+L+1}\).

\paragraph{Information warm-up.} Before using the shrunken block score, \(\algname\) ensures that the maintained curvature-weighted Gram matrix is sufficiently well-conditioned. At the beginning of block \(\tau\), the algorithm checks whether $\lambda_{\min}(\Vb_\tau)\ge \frac{L}{d}$. If this condition fails, the algorithm executes an exploration block and updates the OMD estimator using the observed rewards. For a candidate block \(\xb=(\xb_1,\ldots,\xb_{m+L})\), let $\Ab_i(\xb) := \left( \Ib_d+\sum_{j=1}^{m}\xb_{i-j}\xb_{i-j}^{\top} \right)^{\gamma}$, where for indices \(i-j\le 0\) we interpret \(\xb_{i-j}\) as the corresponding action stored in the current memory state before the block begins.
Define the memory-dependent feature $ \bb_i(\xb):=\Ab_i(\xb)^\top\xb_i$. We then choose the warm-up block $\xb^{\rm warm}$ as follows:
\begin{equation*}
\xb^{\mathrm{warm}}
\in
\argmax_{\xb\in\Xcal^{m+L}}
\lambda_{\min}\!\left(
\sum_{i=m+1}^{m+L}
\bb_i(\xb)\bb_i(\xb)^\top
\right).
\end{equation*}
Since the matrices \(\Vb_\tau\) are monotone increasing, once the condition is satisfied, it remains satisfied in all subsequent blocks. 
We denote by \(\tau_{\mathrm{warm}}\) the number of warm-up blocks. The warm-up condition is used to control the cumulative cost of the block-level safety bonus.

\paragraph{Shrunken confidence bound.}
We now define the shrunken confidence bound used after the information warm-up phase. 
For any block $\xb$, the pre-block action-wise confidence width is
\begin{equation}
    \label{eq:pre-block-ucb}
    \begin{split}
        \gamma_{\tau,i}^{\mathrm{pre}}(\xb) 
        & := \beta \dot\mu(\bb_i(\xb)^\top\thetab_\tau) \|\bb_i(\xb)\|_{\Hb_\tau^{-1}}
        \\
        & \quad + \frac{\nu U_\mu}{2} \beta^2 \|\bb_i(\xb)\|_{\Hb_\tau^{-1}}^2 \, .
    \end{split}
\end{equation}
This width is valid because it is computed from the information available before the block is executed. However, using it for every action ignores the feature information accumulated while the block is being constructed.

To exploit within-block feature information, we follow the idea of the shrinking-confidence construction in linear combinatorial bandits~\citep{jin2021shrinking}, where earlier actions contribute information that reduces the uncertainty of later actions. In linear bandits, this can be done by directly augmenting the Gram matrix with the features of previously selected actions. In generalized linear models, however, the situation is substantially more delicate. The information contributed by an observation is no longer determined solely by its feature vector \(\bb_i(\xb)\); it is weighted by the local curvature \(\dot{\mu}\) of the link function.
Since the parameter estimate is not updated during block construction, this curvature information is not directly available.
We overcome this difficulty by exploiting the geometry of the current confidence set.

If within-block feedback were available, the information accumulated up to action \(i\) would depend on the sequence of intermediate estimators generated during block execution.
Since these estimators are unavailable when evaluating a candidate block, we instead certify the possible curvature values using only the pre-block estimate \(\thetab_\tau\).
By the confidence guarantee for \(\thetab_\tau\), relevant parameters are contained, with high probability, in the enlarged confidence set
\[
    \Ccal_\tau^{(2\beta)}
    :=
    \left\{
        \thetab\in\Bcal_d(C_\theta):
        \|\thetab-\thetab_\tau\|_{\Hb_\tau}\le 2\beta
    \right\};
\]
see Figure~\ref{fig:aux_theta}.
We therefore define auxiliary curvature certificates by maximizing and minimizing the local curvature over \(\Ccal_\tau^{(2\beta)}\).
For each \(i=m+1,\ldots,m+L\),
\begin{align*}
    \thetab_{\tau,i}^{+,\xb}
    &\in
    \argmax_{\thetab\in \Ccal_\tau^{(2\beta)}}
    \dot\mu\!\left(\bb_i(\xb)^\top\thetab\right),
    \\
    \thetab_{\tau,i}^{-,\xb}
    &\in
    \argmin_{\thetab\in \Ccal_\tau^{(2\beta)}}
    \dot\mu\!\left(\bb_i(\xb)^\top\thetab\right).
\end{align*}
The lower certificate provides a conservative estimate of the information that can be safely accumulated within the block, while the upper certificate is used when bounding the uncertainty of the current action.
We now define the conservative within-block information matrix:
\[
   \underline{\Hb}_{\tau,i}(\xb)
   :=
   \Hb_\tau
   +
   \sum_{j=m+1}^{i-1}
   \dot\mu\!\left(\bb_j(\xb)^\top\thetab_{\tau,j}^{-,\xb}\right)
   \bb_j(\xb)\bb_j(\xb)^\top .
\]
Using this matrix and the upper curvature certificate, define the shrunken confidence width
\[
    \begin{split}
        \overline\gamma_{\tau,i}(\xb)
        &:=
        \beta
        \dot\mu\!\left(\bb_i(\xb)^\top\thetab_{\tau,i}^{+,\xb}\right)
        \|\bb_i(\xb)\|_{\underline{\Hb}_{\tau,i}(\xb)^{-1}}
        \\
        &\quad
        +
        \frac{\nu U_\mu}{2}
        \beta^2
        \|\bb_i(\xb)\|_{\underline{\Hb}_{\tau,i}(\xb)^{-1}}^2 .
    \end{split}
\]
Then, the SCB bonus of a block \(\xb\) is
\[
   \Gamma_\tau^{\mathrm{SCB}}(\xb)
   :=
   -
   \sum_{i=m+1}^{m+L}
   \gamma_{\tau,i}^{\mathrm{pre}}(\xb)
   +
   2
   \sum_{i=m+1}^{m+L}
   \overline\gamma_{\tau,i}(\xb) .
\]
This term is the block-level analogue of selecting actions by an improvement over a lower confidence estimate.
It can be smaller than the standard pre-block optimistic bonus because \(\underline{\Hb}_{\tau,i}(\xb)\) includes feature information from actions already placed in the candidate block.

\begin{figure}[t!]
    \centering
    \includegraphics[width=0.8\linewidth]{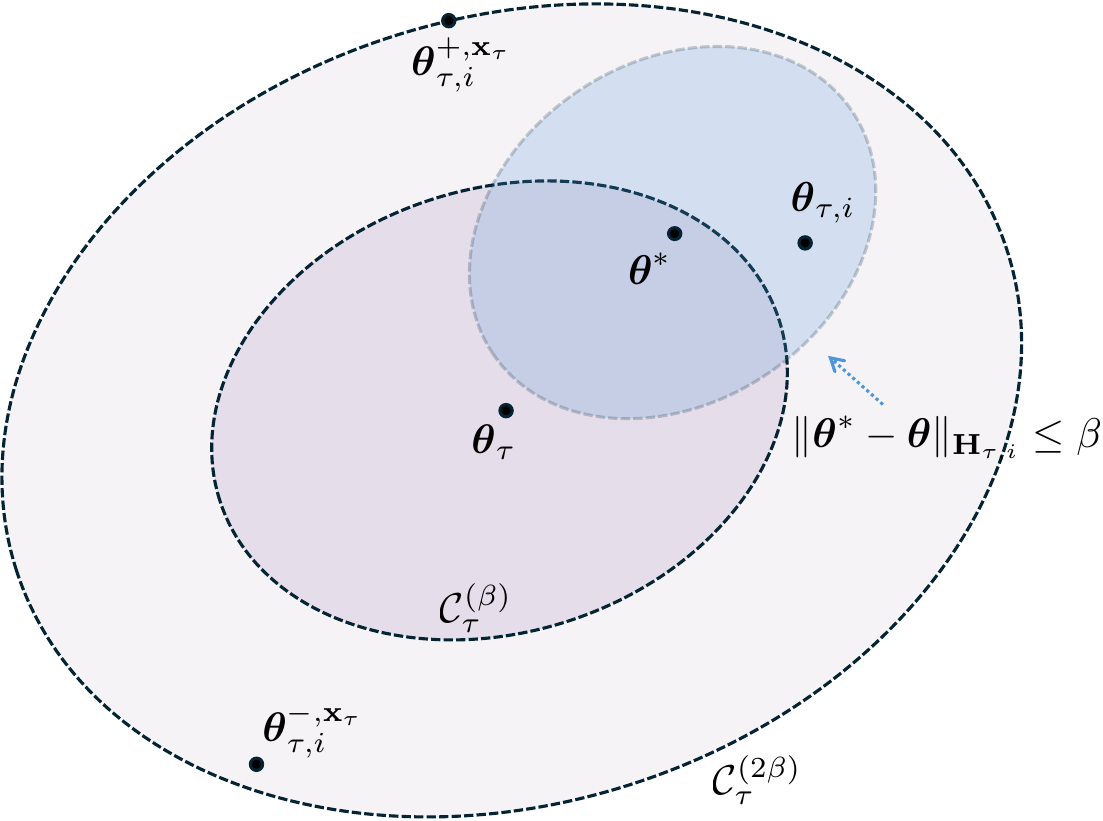}
    \caption{ 
    The purple ellipsoid, centered at \(\thetab_\tau\), illustrates the pre-block confidence region containing \(\thetab^*\).
    The darker blue ellipsoid around \(\thetab^*\) illustrates the concentration region for the within-block estimates.
    Together, these imply that the relevant parameters are covered by the enlarged ellipsoid \(\Ccal_\tau^{(2\beta)}\), which allows curvature to be certified without observing within-block feedback.
    }
    \label{fig:aux_theta}
\end{figure}

\paragraph{Block-safety bonus.}
The shrunken SCB bonus does not necessarily provide optimism for the true block reward. Therefore, we also construct a block-level safety bonus that directly upper bounds the prediction error of the entire proxy reward.
Let $\gb_\tau(\xb) := \sum_{i=m+1}^{m+L} \dot\mu(\bb_i(\xb)^\top\thetab_\tau)\bb_i(\xb)$ and $\Qb(\xb) := \sum_{i=m+1}^{m+L} \bb_i(\xb)\bb_i(\xb)^\top$.
Define \[ \Gamma_\tau^{\mathrm{safe}}(\xb) := \beta\|\gb_\tau(\xb)\|_{\Hb_\tau^{-1}} + \frac{\nu U_\mu}{2} \beta^2 \left\| \Hb_\tau^{-1/2}\Qb(\xb)\Hb_\tau^{-1/2} \right\|_* . \] This bonus treats the entire block as a single decision and certifies the prediction error of the block proxy reward. 

After the information warm-up phase, \(\algname\) selects a block by
\begin{equation*}
   \xb_\tau
   \in
   \argmax_{\xb\in\Xcal^{m+L}}
   \left\{
   \hat r_\tau(\xb)
   +
   \max\left\{
        \Gamma_\tau^{\mathrm{SCB}}(\xb),
        \Gamma_\tau^{\mathrm{safe}}(\xb)
   \right\}
   \right\},
\end{equation*}
where
$
    \hat r_\tau(\xb)
    :=
    \sum_{i=m+1}^{m+L}
    \mu\!\left(\bb_i(\xb)^\top\thetab_\tau\right)
$
is the estimated proxy reward.
Thus, the algorithm uses the shrunken SCB bonus whenever it is sufficient for block-level safety, and otherwise falls back to the block-safety bonus.
This construction preserves the benefit of within-block shrinkage while maintaining a valid block-level optimistic comparison with the proxy-optimal block.

\begin{remark}[Maximization oracle]
    Algorithm~\ref{alg:main algorithm} assumes access to an exact maximization oracle for the block objective, as in the block-optimization framework of~\citet{clerici2024linear}.
    In continuous action domains, the block objective is differentiable, so gradient-based methods provide a natural practical
    approach to approximate block optimization. Extending the regret analysis to such approximate oracles is an important direction for future work.
\end{remark}

\subsubsection{Regret Bound}
\begin{restatable}[Regret bound of $\algname$]{theorem}{ThmRegretBoundofGLMBSCB}
    \label{thm:regret main algorithm}
    Suppose that Assumptions~\ref{assm:bounded domain}, \ref{assm:bounded link}, and~\ref{assm:self concordance} hold.
        If Algorithm~\ref{alg:main algorithm} is run with
        \(\beta=\Ocal(\nu R C_{\theta}\sqrt{d\nu R C_{\theta}^{2}+d\log(RT)})\),
        \(\lambda=\Ocal(d\nu^3R^3C_\theta)\), and
        \(L=\lfloor\sqrt{mT}\rfloor-m\), then it satisfies
    \begin{equation*}
        \regret_T
        =
        \widetilde\Ocal\!\left(
            \sqrt{mT}
            +
            d\sqrt T
            +
            \sqrt\kappa\,d^2m^{1/4}T^{1/4}
            +
            \kappa d^2
        \right).
    \end{equation*}
\end{restatable}

\paragraph{Discussion of Theorem~\ref{thm:regret main algorithm}.} 
Theorem~\ref{thm:regret main algorithm} gives a
\(\widetilde{\Ocal}(\sqrt T)\) regret bound for generalized linear bandits with memory.
To the best of our knowledge, this is the first such guarantee for action-induced endogenous non-stationary generalized linear bandits.
Importantly, the leading term is independent of the curvature parameter \(\kappa\), and the algorithm does not require prior knowledge of \(\kappa\), unlike many existing methods for non-stationary generalized linear bandits~\citep{wang2023revisiting}.
In the linear reward case, the proposed algorithm achieves
\(\widetilde{\Ocal}(\sqrt{mT}+d\sqrt T)\) regret.
This further improves over the sharpened \(\widetilde{\Ocal}(\sqrt{dmT}+d\sqrt T)\) guarantee for \texttt{OFUL-Memory} in Theorem~\ref{thm:improved regret for linear model informal}, showing that the shrunken-confidence design separates the dependence on \(m\) and \(d\) even when the reward is linear.
Detailed proofs are provided in Appendix~\ref{appx:proof for main algorithm}.

\begin{remark}
Algorithm~\ref{alg:main algorithm} maximizes the reward in the nonstationary setting when the model parameters $m$ and $\gamma$ are known. When this information is not available, we can adapt the bandit combiner approach \citep{cutkosky2020upper, clerici2024linear} to our algorithm. The bandit combiner takes models with different configurations $(m, \gamma)$ as inputs, treats each model as an arm, and selects a model at each round using a UCB strategy. The selected model then executes one block of actions according to its configuration. In our setting, since the block size depends on the model, we use the average reward within a block as the feedback signal. As rounds progress, poorly performing models are gradually eliminated, allowing the algorithm to identify the appropriate configuration $(m, \gamma)$.
\end{remark}

\subsubsection{Regret Analysis}\label{subsec:regret analysis}
\paragraph{Confidence sets.}
The OMD estimator provides a high-probability confidence set for the unknown parameter.
By adapting the OMD-based confidence bound for GLMs~\citep{zhang2025generalized} to our block-structured setting with memory-dependent features, we obtain the following prediction-error bound, which underlies the pre-block action-wise widths used in in Eq.~\eqref{eq:pre-block-ucb}.
\begin{restatable}[Pre-block prediction-error bound]{proposition}{PreblockPredictionErrorBound}\label{prop:exp_bound}

Suppose that the estimator \(\thetab_\tau\) satisfies
\(\|\thetab_\tau-\thetab^*\|_{\Hb_\tau}\le \beta\).
Then, for any candidate block
\(\xb=(\xb_1,\ldots,\xb_{m+L})\in\Xcal^{m+L}\) and any
\(i \in \{m+1,\ldots,m+L\}\), letting
\(\bb_i(\xb):=\Ab_i(\xb)^\top\xb_i\), we have
\begin{align*}
    &
    \left|
    \mu\!\left(\bb_i(\xb)^\top\thetab_\tau\right)
    -
    \mu\!\left(\bb_i(\xb)^\top\thetab^*\right)
    \right|
    \le \gamma_{\tau,i}^{\mathrm{pre}}(\xb)  \, .
\end{align*}    
\end{restatable}
For a single action, this bound gives a valid optimistic confidence width.
However, \(\algname\) does not simply sum these pre-block widths over all actions in a block.
Instead, it uses a shrunken block bonus that incorporates feature information from actions already placed in the candidate block.
This shrinkage is statistically useful, but it introduces a new proof challenge: \textit{the shrunken bonus alone does not necessarily upper bound the true block reward}.

\begin{figure*}[t!]
    \centering
    \begin{subfigure}[t]{0.32\textwidth}
        \centering
        \includegraphics[width=0.95\linewidth]{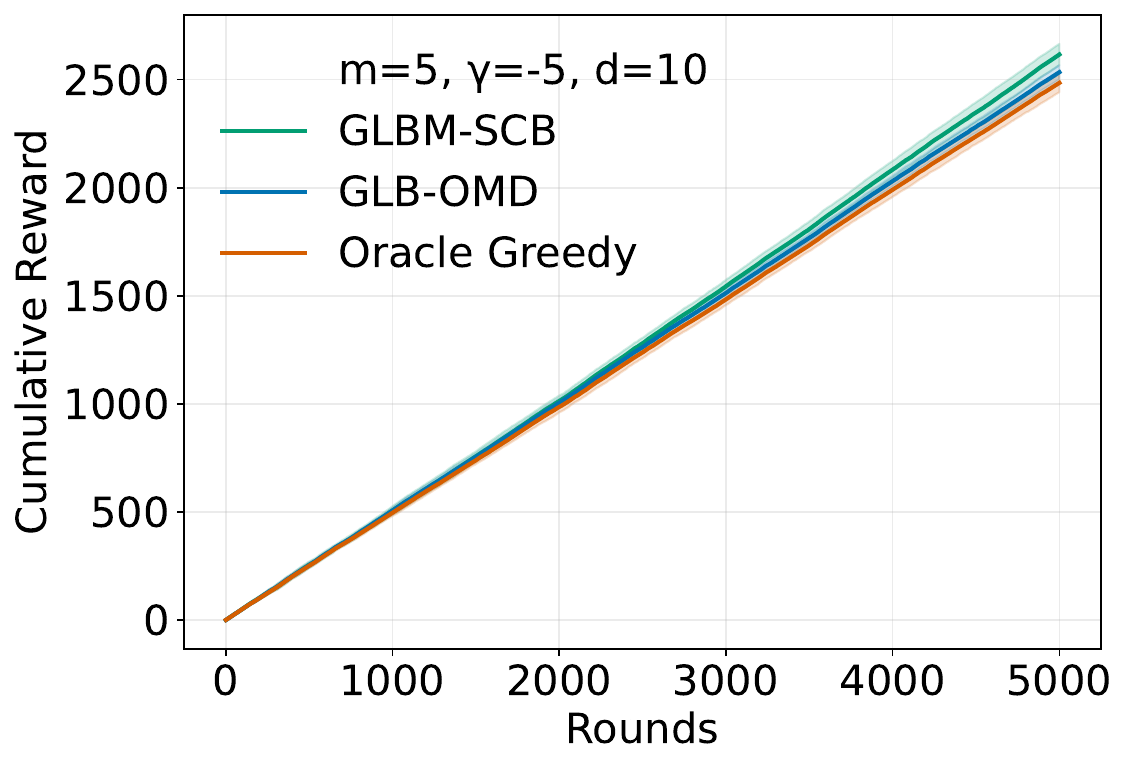}
        \caption{Logistic reward with rotting env.}
        \label{fig:logistic-rotting}
    \end{subfigure}
    \hfill
    \begin{subfigure}[t]{0.32\textwidth}
        \centering
        \includegraphics[width=0.95\linewidth]{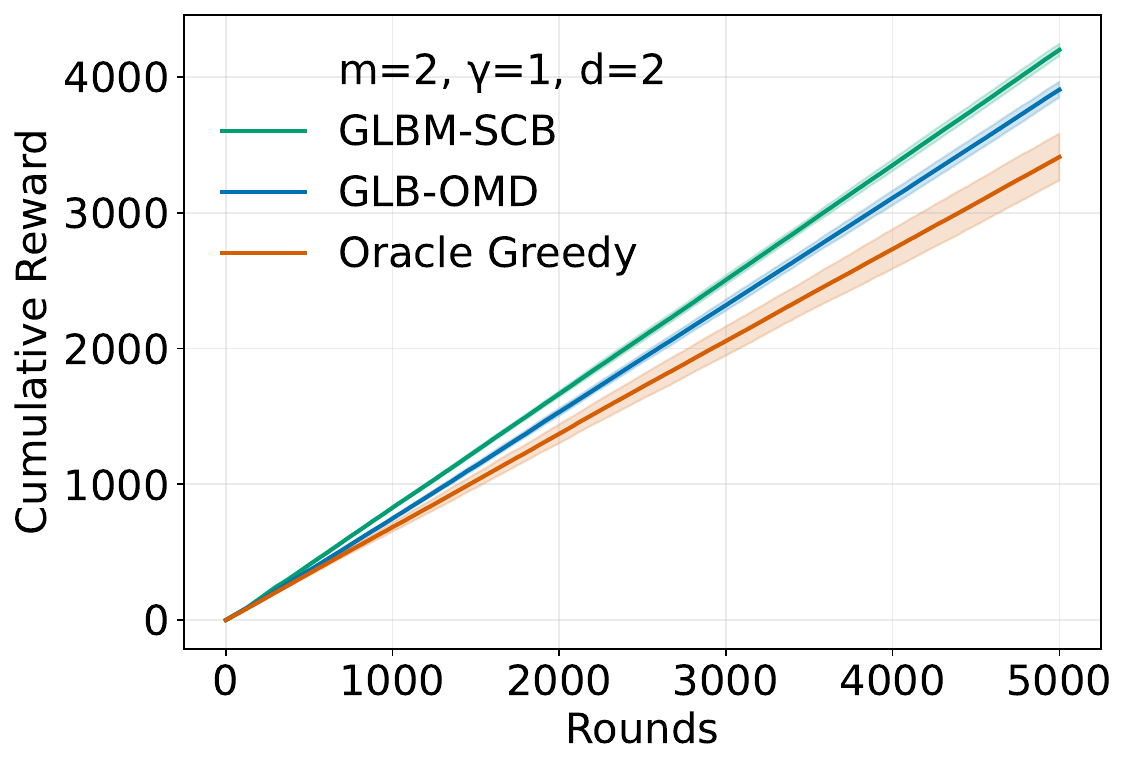}
        \caption{Logistic reward with rising env.}
        \label{fig:logistic-rising}
    \end{subfigure}
    \hfill
    \begin{subfigure}[t]{0.32\textwidth}
        \centering
        \includegraphics[width=0.95\linewidth]{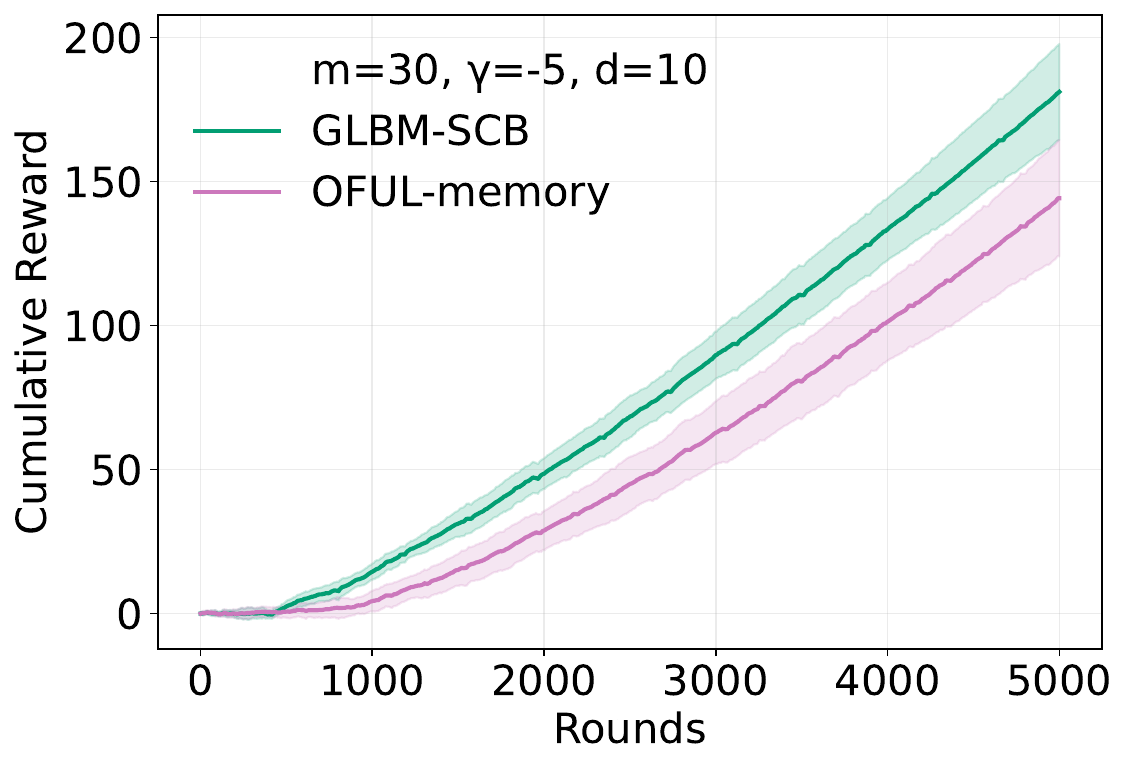}
        \caption{Linear reward with rotting env.}
        \label{fig:linear-rotting}
    \end{subfigure}

    \caption{
    Cumulative rewards in bandits-with-memory environments.
    }
    \label{fig:env}
\end{figure*}

\paragraph{Block-level safety and selected-block decomposition.}
To recover a valid comparison with the proxy-optimal block, \(\algname\) uses the block-safety bonus \(\Gamma_\tau^{\mathrm{safe}}\).
For any candidate block \(\xb\), a second-order Taylor expansion of the block proxy reward around \(\thetab_\tau\) shows that, on the confidence event \(\|\thetab_\tau-\thetab^*\|_{\Hb_\tau}\le \beta\),
\[
    \big|
    \widetilde r(\xb)-\hat r_\tau(\xb)
    \big|
    \le
    \Gamma_\tau^{\mathrm{safe}}(\xb).
\]
Thus, \(\widetilde r(\xb)\le \hat r_\tau(\xb)+\Gamma_\tau^{\mathrm{safe}}(\xb)\).
Since \(\algname\) selects a block using the score
\(\hat r_\tau(\xb)+\max\{\Gamma_\tau^{\mathrm{SCB}}(\xb),\Gamma_\tau^{\mathrm{safe}}(\xb)\}\), the selection rule is optimistic at the block level.

Let \(\overline\Gamma_\tau(\xb):=\sum_{i=m+1}^{m+L}\overline\gamma_{\tau,i}(\xb)\) and
\(\Gamma_\tau^{\mathrm{pre}}(\xb):=\sum_{i=m+1}^{m+L}\gamma_{\tau,i}^{\mathrm{pre}}(\xb)\), so that
\(\Gamma_\tau^{\mathrm{SCB}}(\xb)=-\Gamma_\tau^{\mathrm{pre}}(\xb)+2\overline\Gamma_\tau(\xb)\).
Then the block selection rule of \(\algname\) yields the decomposition
\[
    \widetilde r(\widetilde\xb)
    -
    \widetilde r(\xb_\tau)
    \le
    2\overline\Gamma_\tau(\xb_\tau)
    +
    \xi_\tau(\xb_\tau),
\]
where 
\(\xi_\tau(\xb):=[\Gamma_\tau^{\mathrm{safe}}(\xb)-\Gamma_\tau^{\mathrm{SCB}}(\xb)]_+\) is the safety correction term.
This decomposition replaces the standard UCB regret argument.
Rather than requiring the shrunken score to be optimistic action by action, the analysis restores optimism at the block level through \(\Gamma_\tau^{\mathrm{safe}}\), and then pays only the shrunken widths and the safety correction.

\paragraph{Bounding the safety correction.}
It remains to control the cumulative safety correction
\(\sum_\tau \xi_\tau(\xb_\tau)\).
The correction \(\xi_\tau\) measures how much the SCB bonus falls short of the block-safety bonus.
The information warm-up condition
ensures that the maintained Gram matrix is sufficiently well-conditioned before the block-safe SCB rule is used.
This spectral lower bound allows us to control the difference between the pre-block widths and the shrunken within-block widths by eigenvalue-telescoping argument.
As a result,
\[
    \sum_{\tau>\tau_{\mathrm{warm}}}
    \xi_\tau(\xb_\tau)
    =
    \widetilde\Ocal\!\left(
        \sqrt {\kappa L}d^2+
        \kappa d^2
    \right).
\]
Thus, the safety correction is lower order and does not affect the leading \(\sqrt T\) term.

Combining the warm-up cost, the cyclic approximation error, and the selected-block regret bound gives
\[
    \widetilde{\Ocal} \bigg(
        \underbrace{(m+L)\tau_{\mathrm{warm}}}_{\text{warm-up cost}}
        +
        \underbrace{\frac{mT}{m+L} + L}_{\text{approximation error}}
        +
        \underbrace{\Ecal_{\mathrm{SCB}}(L,T)}_{\text{estimation error}}
    \bigg)
\]
regret.
The analysis of the block-safe SCB rule gives
$
    \Ecal_{\mathrm{SCB}}(L,T)
    =
    \widetilde{\Ocal}\!\left(
        d\sqrt T
        +
        \sqrt{\kappa L}\,d^2
        +
        \kappa d^2
    \right).
$
Choosing \(L=\lfloor\sqrt{mT}\rfloor-m\) balances the approximation terms.
If \(\tau_{\mathrm{warm}}=\widetilde{\Ocal}(1)\), the warm-up cost is absorbed into the leading \(\sqrt{mT}\) term.
The full proof is deferred to Appendix~\ref{appx:proof for main algorithm}.

\section{Experiments} \label{sec:experiments}
Our experiments are designed to illustrate three phenomena: (i) in nonlinear reward models, accounting for memory effects improves over myopic or non-memory-aware baselines; (ii) in rising environments, planning over blocks is crucial because actions can improve the future reward landscape; and (iii) even in the linear reward case, the proposed shrunken-confidence design improves over the \texttt{OFUL-memory} baseline.

We compare \(\algname\) with two baselines. The first is \texttt{GLB-OMD}~\citep{zhang2025generalized}, a state-of-the-art generalized linear bandit algorithm that ignores memory effects and selects actions on a round-by-round basis as if rewards depended only on the current action.
The second is \texttt{Oracle Greedy}, which has access to the true parameter \(\thetab^*\) and selects an action greedily with respect to the current expected reward. For the linear reward experiment, we additionally compare with \texttt{OFUL-memory}~\citep{clerici2024linear}.

Figure~\ref{fig:env} reports cumulative rewards across scenarios.
In the logistic rotting environment (Figure~\ref{fig:logistic-rotting}), \(\algname\) consistently outperforms both \texttt{GLB-OMD} and \texttt{Oracle Greedy}, and the gap widens over time.
This behavior highlights the benefit of optimizing over entire blocks rather than making myopic round-by-round decisions.
In the logistic rising environment (Figure~\ref{fig:logistic-rising}), the performance gap becomes more pronounced: Oracle Greedy fails to account for the long-term effect of actions on future rewards, whereas \(\algname\) benefits from block-level planning.
Finally, in a linear rotting environment with a larger memory length (Figure~\ref{fig:linear-rotting}), \(\algname\) achieves higher cumulative reward than \texttt{OFUL-memory}.
This is consistent with our theory, which predicts that the SCB design separates the leading dependence on the memory length \(m\) and the feature dimension \(d\).
The details are deferred to Appendix~\ref{appx:experiment-details}.

\section{Conclusion}
We studied generalized linear bandits with memory, where rewards depend on a finite history of past actions. 
Our results show that memory-induced non-stationarity does not, by itself, fundamentally increase the statistical complexity of the problem. 
In particular, the previously known \(\widetilde \Ocal(T^{3/4})\) regret bound for linear bandits with memory arises from an analytical mismatch between block-level estimation and action-wise uncertainty control.
By treating each block as a single combinatorial decision unit, we recover a sharper \(\widetilde \Ocal(\sqrt T)\) regret bound.
For generalized linear rewards, we introduce a SCB-based exploration scheme based on two complementary bonuses: a shared block-level bonus computed from past blocks and an inflated intra-block bonus that accounts for the actions already selected within the current block. 
This design controls yields regret of order \(\widetilde \Ocal(\sqrt{mT}+\sqrt{dT})\), thereby decoupling the memory length from
the feature dimension.
A natural direction for future work is to relax the exact UCB oracle used for block optimization.
Developing efficient approximate oracles, and quantifying how their optimization error propagates into regret, would make the framework more computationally practical and applicable to richer action classes.

\section*{Impact Statement}
This paper presents work whose goal is to advance the field of 
Machine Learning. There are many potential societal consequences 
of our work, none which we feel must be specifically highlighted here.

\section*{Acknowledgments}
This work was supported by Shinsegae and by the National Research Foundation of Korea~(NRF) grant and the Institute of Information \& communications Technology Planning \& Evaluation~(IITP) grant both funded by the Korea government~(MSIT) (No. RS-2022-NR071853, RS-2023-00222663, RS-2025-25463302, RS-2026-25507282).

\bibliography{references}
\bibliographystyle{icml2026}

\newpage
\appendix
\onecolumn

\renewcommand{\contentsname}{Contents of Appendix}
\addtocontents{toc}{\protect\setcounter{tocdepth}{1}}
{
  \hypersetup{hidelinks}
  \tableofcontents
}

\section{Additional Notations}

We define the sequence $\bb_\tau = (\bb_{\tau,i})_{i=1}^{m+L}$ for a selected action block $\xb_\tau=(\xb_{\tau,1}, \ldots, \xb_{\tau,m+L})$ as
\begin{equation*}
    \bb_{\tau,i}:=
    \begin{cases}
        \xb_{\tau,i}, & 1 \le i \le m, \\
        {\Ab}_{\tau, i-1}^\top\,\xb_{\tau,i}, & m+1 \le i \le m+L ,
    \end{cases}
\end{equation*}
where ${\Ab}_{\tau, i-1}=\left(\Ib_d+\sum_{s=1}^{m}\,\xb_{\tau, i-s} \xb_{\tau, i-s}^{\top}\right)^{\gamma}$.
Similarly, for the proxy optimal block $   \tilde{\xb}
    \in
    \argmax_{\xb\in\Xcal^{m+L}} \tilde r(\xb)$, we define $\tilde{\bb} = (\tilde{\bb}_{i})_{i=1}^{m+L}$ as
\begin{equation*}
    \tilde{\bb}_{i}:=
    \begin{cases}
        \tilde{\xb}_i, & 1 \le i \le m, \\
        \tilde{\Ab}_{i-1}^\top\, \tilde{\xb}_i, & m+1 \le i \le m+L ,
    \end{cases}
\end{equation*}
where $\tilde{\Ab}_{i-1}=\left(\Ib_d+\sum_{s=1}^{m}\tilde{\xb}_{i-s} \tilde{\xb}_{i-s}^{\top}\right)^{\gamma}$.

\section{Related Work}
\paragraph{Linear bandits with memory.}
Linear bandit with memory is a model to capture user's nonstationary preference behavior, introduced by \citet{clerici2024linear}. In this setting, the expected reward at time step $t$ becomes $r_t = \langle \xb_t, \Ab_{t-1}\thetab \rangle$, where $\Ab_{t-1}$ is a memory matrix defined as $(\Ib_d + \sum_{s=1}^{m}\xb_{t-s}\xb_{t-s}^\top)^\gamma$. Here, $m$ denotes the memory size that determines the range of past actions influencing the current reward, and $\gamma \in \RR$ quantifies the rotting or rising effect of past actions. They address nonstationarity by converting the problem into a stationary bandit setting through a carefully designed cyclic approximation. Consequently, their $\texttt{OFUL-Memory}$ algorithm achieves a regret bound of $\widetilde{\Ocal}\big(\sqrt{d}(m+1)^{{\frac{1}{2}}+\max\{\gamma,~0\}}T^{3/4}\big)$, which is suboptimal compared to the best known general lower bound of $\Omega(d\sqrt{T})$. 

\paragraph{Generalized linear bandits. } 
Generalized linear bandits (GLBs) were first studied by \citet{filippi2010parametric}, who proposed \texttt{GLM-UCB}.
Subsequent works have focused on improving statistical efficiency in GLBs, primarily through maximum likelihood estimation, including \citet{lee2024a, sawarni2024generalized, liu2024almost, clerico2025confidence}.
For the logistic bandits which is a special case of GLBs, several works have developed the confidence set based on Online Mirror Descent (OMD)-based online estimators \citep{faury2022jointly, zhang2023online, lee2024nearly, lee2025improved}.
More generally, \citet{zhang2025generalized} developed an OMD-based online estimator for generalized linear bandits and derived refined confidence sets with radius $\mathcal{O}(\sqrt{d\log t})$.
We adopt their confidence sets and online parameter updates in the generalized linear bandit setting with memory.

\paragraph{Combinatorial bandits. } In the combinatorial bandit framework~\citep{qin2014contextual, jin2021shrinking, hwang2023combinatorial, liu2025combinatorial}, an agent selects a set of arms rather than a single arm at each round. Thus, it is natural to consider techniques from combinatorial bandits for bandit problems with memory, since algorithms in both settings select a set of arms without updating parameters. In particular, our work is related to combinatorial logistic bandits~\citep{liu2025combinatorial}, where expected rewards follow a logistic model. However, while their algorithm involves a complicated projection step and requires prior knowledge of $\kappa$, our algorithm is substantially simpler and employs a different confidence set that does not require knowledge of $\kappa$.
Moreover, most existing works on combinatorial bandits (including \citet{liu2025combinatorial}) employ a regularization parameter that scales with the block length $K$ to obtain statistical guarantees, which results in a regret bound of $\widetilde{\Ocal}(K\sqrt{dT})$ and becomes unfavorable when $K$ is large. \citet{jin2021shrinking} highlight this limitation in combinatorial bandits with large set sizes and propose an algorithm that updates its strategy within a block without performing parameter estimation. However, their analysis relies on a linear reward model, and whether their approach can be extended to the generalized linear model setting has remained an open problem.

\section{Canonical Exponential Family} \label{appx:canonical-family}
We provide a canonical exponential-family specification that induces the conditional mean model in Eq.~\eqref{eq:glm-memory-mean}. Conditioned on \((\Fcal_{t-1},\xb_t)\), let the reward \(y_t\) follow a canonical exponential-family distribution with density 
\begin{align*}
    \PP(y_t \mid \Fcal_{t-1}, \xb_t) = \exp \left( \frac{y_t z_t-\psi(z_t)}{g(\phi)} + h(y_t,\phi) \right) \, , 
\end{align*}
where $z_t := \langle \xb_t,\Ab_{t-1}\thetab^*\rangle$.
Here, \(\phi\in\RR_+\) is a known scale parameter, \(\psi(\cdot)\) is the log-partition function, \(g(\phi)>0\) is a known dispersion parameter, and \(h(\cdot,\phi)\) is the base measure.

Under this model, standard properties of canonical exponential families give \[ \mathbb E[y_t \mid \Fcal_{t-1},\xb_t] = \dot\psi(z_t), \qquad \mathrm{Var}(y_t \mid \Fcal_{t-1},\xb_t) = g(\phi)\ddot\psi(z_t). \] 
We define the link function by $\mu(z):=\dot\psi(z)$ so that \(\dot\mu(z)=\ddot\psi(z)\). Hence, we have 
\begin{equation*}
    \mathbb E[y_t \mid \Fcal_{t-1}, \xb_t] = \mu\!\left(\langle \xb_t,\Ab_{t-1}\thetab^*\rangle\right)
    \, ,
    \quad 
    \mathrm{Var}(y_t \mid \Fcal_{t-1}, \xb_t) = g(\phi) \dot\mu\!\left(\langle \xb_t,\Ab_{t-1}\thetab^*\rangle\right) \, .
\end{equation*}

\section{Proof of Proposition~\ref{prop:failure-of-oracle-greedy}}
\FailureOfOracleGreedy*
\begin{proof}[Proof of Proposition~\ref{prop:failure-of-oracle-greedy}]
    We prove the statement by constructing rotting and rising instances in which the
    oracle greedy policy suffers linear regret. Throughout the proof, we compare the
    oracle greedy policy with a fixed reference policy. Since the optimal policy
    achieves at least the cumulative reward of any reference policy, a linear gap
    between the reference policy and oracle greedy implies linear regret.
    
    \paragraph{Rotting case.}
    Let \(d=2\), \(\Xcal=\{\eb_1,\eb_2\}\), \(m=1\), \(\thetab^*=\eb_1\), and
    \(\gamma=-\beta\) for some \(\beta>0\) to be chosen later. Let the initial memory
    matrix be \(\Ab_0=\Ib_2\), and for \(t\ge2\),
    $
        \Ab_{t-1}
        =
        \left(\Ib_2+\xb_{t-1}\xb_{t-1}^{\top}\right)^{-\beta} \, .
    $
    Since \(\mu\) is continuous and strictly increasing, we have
    \[
        \lim_{\beta\to\infty}\mu(2^{-\beta})=\mu(0)
        \quad\text{and}\quad
        \mu(1)>\mu(0).
    \]
    Hence, for sufficiently large \(\beta>0\),
    \[
        \mu(2^{-\beta})
        <
        \frac{\mu(1)+\mu(0)}{2}.
    \]
    Fix such a \(\beta\).
    We first characterize the oracle greedy policy. At \(t=1\), since
    \[
        \langle \eb_1, \Ab_0\thetab^*\rangle = 1,
        \qquad
        \langle \eb_2, \Ab_0\thetab^*\rangle = 0,
    \]
    and \(\mu\) is strictly increasing, the greedy policy selects \(\eb_1\).
    Suppose now that the greedy policy selected \(\eb_1\) at time \(t-1\). Then
    \[
        \Ab_{t-1}
        =
        \left(\Ib_2+ \eb_1 \eb_1^\top\right)^{-\beta}
        =
        \begin{pmatrix}
            2^{-\beta} & 0 \\
            0 & 1
        \end{pmatrix}.
    \]
    Therefore,
    \[
        \langle \eb_1,\Ab_{t-1}\thetab^*\rangle = 2^{-\beta},
        \qquad
        \langle \eb_2,\Ab_{t-1}\thetab^*\rangle = 0.
    \]
    Again by strict monotonicity of \(\mu\), the greedy policy selects \(\eb_1\).
    By induction, oracle greedy selects \(\eb_1\) at every round. Hence its expected
    cumulative reward is
    \[
        R_T^{\mathrm{greedy}}
        =
        \mu(1)+(T-1)\mu(2^{-\beta}).
    \]
    
    Now consider the reference policy that alternates between \(\eb_1\) and \(\eb_2\),
    starting with \(\eb_1\). Under this policy, every two consecutive rounds yield
    expected reward \(\mu(1)+\mu(0)\). Thus, for all \(T\), the expected cumulative reward of the reference policy satisfies
    \[
        R_T^{\mathrm{ref}}
        \ge
        \frac{\mu(1)+\mu(0)}{2} T \, .
    \]
    Consequently,
    \[
    \begin{aligned}
        \regret_T(\pi^{\mathrm{greedy}})
        &=
        \mathrm{OPT}(T)-R_T^{\mathrm{greedy}}
        \\
        &\ge
        R_T^{\mathrm{ref}}-R_T^{\mathrm{greedy}}
        \\
        &\ge
        \frac{\mu(1)+\mu(0)}{2} T
        -
        \mu(1)
        -
        (T-1)\mu(2^{-\beta})
        \\
        & 
        \ge \left( \frac{\mu(1)+\mu(0)}{2} - \mu(2^{-\beta}) \right) T - \mu(1) + \mu(2^{-\beta}) \, .
    \end{aligned}
    \]
    Since 
    $
        \frac{\mu(1)+\mu(0)}{2}
        -
        \mu(2^{-\beta})
        >0,
    $
    , oracle greedy suffers linear regret in the rotting case.

    \paragraph{Rising case.}
    We next construct a rising instance. Let \(d=2\), \(\Xcal=\{\eb_1,\eb_2\}\),
    \(m=1\), and \(\gamma=1\). We use the memory matrix
    \[
        \Ab_{t-1}
        =
        \Ab_0+\xb_{t-1}\xb_{t-1}^{\top},
        \qquad
        \Ab_0=\eb_1\eb_1^\top,
    \]
    with initial memory matrix \(\Ab_0\). Let
    $\thetab^* = (\epsilon, \sqrt{1-\epsilon^2})^\top$,
    where \(0<\epsilon<1/\sqrt{5}\). This choice ensures that
    $
        2\epsilon < \sqrt{1-\epsilon^2} \, .
    $
    
    At \(t=1\), we have
    \[
        \langle \eb_1, \Ab_0\thetab^*\rangle = \epsilon,
        \qquad
        \langle \eb_2, \Ab_0\thetab^*\rangle = 0.
    \]
    Thus oracle greedy selects \(\eb_1\). Suppose that the greedy policy selected
    \(\eb_1\) at time \(t-1\). Then
    \[
        \Ab_{t-1}
        =
        \eb_1\eb_1^\top+\eb_1\eb_1^\top
        =
        2\eb_1\eb_1^\top,
    \]
    and hence
    \[
        \langle \eb_1,\Ab_{t-1}\thetab^*\rangle = 2\epsilon,
        \qquad
        \langle \eb_2,\Ab_{t-1}\thetab^*\rangle = 0.
    \]
    Since \(\mu\) is strictly increasing, greedy again selects \(\eb_1\). Therefore,
    by induction, oracle greedy selects \(\eb_1\) at every round, and its expected
    cumulative reward is
    \[
        R_T^{\mathrm{greedy}}
        =
        \mu(\epsilon)+(T-1)\mu(2\epsilon).
    \]
    
    Now consider the reference policy that always selects \(\eb_2\). At \(t=1\), its
    expected reward is \(\mu(0)\). For every \(t\ge2\), since the previous action is
    \(\eb_2\),
    \[
        \Ab_{t-1}
        =
        \eb_1\eb_1^\top+\eb_2\eb_2^\top
        =
        \Ib_2.
    \]
    Therefore,
    $
        \langle \eb_2,\Ab_{t-1}\thetab^*\rangle
        =
        \sqrt{1-\epsilon^2}.
    $
    The expected cumulative reward of this reference policy is
    \[
        R_T^{\mathrm{ref}}
        =
        \mu(0)+(T-1)\mu\!\left(\sqrt{1-\epsilon^2}\right).
    \]
    Thus,
    \[
    \begin{aligned}
        \regret_T(\pi^{\mathrm{greedy}})
        &\ge
        R_T^{\mathrm{ref}}-R_T^{\mathrm{greedy}}
        \\
        &=
        (T-1)
        \left[
            \mu\!\left(\sqrt{1-\epsilon^2}\right)
            -
            \mu(2\epsilon)
        \right]
        +
        \mu(0)-\mu(\epsilon).
    \end{aligned}
    \]
    Since \(2\epsilon<\sqrt{1-\epsilon^2}\) and \(\mu\) is strictly increasing,
    $
        \mu\!\left(\sqrt{1-\epsilon^2}\right)
        >
        \mu(2\epsilon) \, .
    $
    Therefore the regret is \(\Omega(T)\). This proves that oracle greedy can suffer
    linear regret in both rotting and rising scenarios.

\end{proof}

\section{Proof of Proposition~\ref{prop:exp_bound}}
\PreblockPredictionErrorBound*
\begin{proof}[Proof of Proposition~\ref{prop:exp_bound}]
Let $z_{\tau,i} := \bb_i(\xb)^\top\thetab_\tau$ and $h := \bb_i(\xb)^\top(\thetab^* - \thetab_\tau)$. By Taylor's theorem with integral remainder, expanded around $z_\tau$, we have
\begin{align*}
    \mu(\bb_i(\xb)^\top\thetab^*)-\mu(\bb_i(\xb)^\top\thetab_\tau)
    &=
    \mu(z_{\tau,i}+h)-\mu(z_{\tau,i}) \\
    &=
    \dot\mu(z_{\tau,i})h
    +
    h^2\int_0^1 (1-s)\ddot\mu(z_{\tau,i}+s h)\,ds .
\end{align*}
Taking absolute values and using $\dot\mu(z_{\tau,i})\ge 0$ and
$|\ddot\mu(z)|\le \nu U_\mu$, we obtain
\begin{align*}
    \left|
    \mu(\bb_i(\xb)^\top\thetab^*)-\mu(\bb_i(\xb)^\top\thetab_\tau)
    \right|
    &\le
    \dot\mu(z_{\tau,i})|h|
    +
    h^2\int_0^1(1-s)|\ddot\mu(z_{\tau,i}+s h)|\,ds \\
    &\le
    \dot\mu(z_{\tau,i})|h|
    +
    \nu U_\mu h^2\int_0^1(1-s)\,ds \\
    &=
    \dot\mu(z_\tau)|h|
    +
    \frac{\nu U_\mu}{2}h^2 .
\end{align*}
It remains to bound $h$. Since $\Hb_\tau\succ0$, the weighted
Cauchy--Schwarz inequality gives
\begin{align*}
    |h|
    =
    |\bb_i(\xb)^\top(\thetab^*-\thetab_\tau)| 
    \le
    \|\thetab^*-\thetab_\tau\|_{\Hb_\tau}\|{\bb}_i(\xb)\|_{\Hb_\tau^{-1}} \le \beta \|{\bb}_i(\xb)\|_{\Hb_\tau^{-1}}
\end{align*}
Substituting this inequality into the Taylor bound yields
\begin{align*}
    \left|
    \mu(\bb_i(\xb)^\top\thetab_\tau)
    -
    \mu(\bb_i(\xb)^\top\thetab^*)
    \right|
    \le{}
    \beta
    \dot{\mu}\!\left(\bb_i(\xb)^\top\thetab_{\tau}\right)
    \|{\bb}_i(\xb)\|_{\Hb_\tau^{-1}}
+
    \frac{\nu U_\mu}{2}
    \beta^2
    \|{\bb}_i(\xb)\|_{\Hb_\tau^{-1}}^2.
\end{align*}
\end{proof}

\section{Proof of Theorem~\ref{thm:improved regret for linear model informal}} \label{appx:proof for linear case}

In this section, we provide the formal proof of Theorem~\ref{thm:improved regret for linear model informal}. Before proceeding, we first state the formal version of Theorem~\ref{thm:improved regret for linear model informal}.

\begin{algorithm}[t!]
    \caption{OFUL-Memory (\texttt{OM})~\citep{clerici2024linear}}
    \label{alg:OM}
    \begin{algorithmic}[1]
        \STATE \textbf{Input:} regularization parameter $\lambda$, confidence radius $\{\beta_\tau\}_{\tau\ge1}$, block length $L$
        \STATE Initialize $\Ccal_1 = \{ \thetab : \|\thetab\|_2 \le C_\theta\}, \Vb_1 = \lambda \Ib_d$
        \FOR{$\tau = 1, 2, \ldots, \lfloor T/(m+L) \rfloor$}
            \STATE Calculate
            \begin{equation} \label{eq:OM action}
                \xb_\tau := (\xb_{\tau,1}, \ldots, \xb_{\tau, m+L}) = \argmax_{ \{ \xb_{\tau, i} \}_{i=1}^{m+L} \subset \Xcal} \sup_{\thetab\in \mathcal{C}_{\tau}}
                \sum_{i=m+1}^{m+L} \langle \xb_{\tau,i}, \Ab_{\tau,i-1} \thetab \rangle
            \end{equation}
            \STATE Play $\xb_\tau$ and collect $y_{\tau,1}, \ldots, y_{\tau, m+L}$
            \STATE Update
            \begin{align*}
                & \Vb_{\tau+1} = \Vb_{\tau} + \sum_{i=m+1}^{m+L} (\Ab_{\tau, i-1} \xb_{\tau, i}) (\Ab_{\tau, i-1} \xb_{\tau, i})^\top
                \\
                & \hat{\thetab}_{\tau+1} = \Vb_{\tau+1}^{-1} \left( \sum_{s=1}^{\tau}  \sum_{i=m+1}^{m+L} y_{s,i} \Ab_{s, i-1} \xb_{s, i} \right)
                \\
                & \Ccal_{\tau+1} = \left\{ \thetab \in \RR^d : \| \hat{\thetab}_{\tau+1} - \thetab \|_{\Vb_{\tau+1}} \le \beta_{\tau+1} \right\}
            \end{align*}
        \ENDFOR
    \end{algorithmic}
\end{algorithm}

\setcounter{theorem}{0}
\begin{theorem}[Improved Regret of~\texttt{OFUL-memory}~\citep{clerici2024linear}] \label{thm:improved regret for linear model formal}
    Suppose that Assumption~\ref{assm:bounded domain} holds and the noise term $\eta_t:=y_t - \langle\xb_t, \Ab_{t-1}\thetab^*\rangle$ satisfies conditionally $\sigma$-sub-gaussian. If we set the algorithmic parameters in Algorithm~\ref{alg:OM} as follows:
    $\lambda=\max\{d, R^2 \}$ ($R:=(1+m)^{\max\{\gamma,0\}}$),  $\beta_\tau=\sqrt{ 2\log T + d \log \left(1+{\frac{\tau(m+L)(m+1)^{2\max\{\gamma,0\}}}{d\lambda}}\right)}+\sqrt{\lambda}C_\theta$, and $L = \lfloor m^{1/2} d^{-1/2}T^{1/2} -m \rfloor$, 
    then \texttt{OFUL-memory} (Algorithm~\ref{alg:OM}) achieves the following cumulative regret:
    \begin{equation*}
        \regret_T = \widetilde{\Ocal}\left(\sigma (m+1)^{\max\{\gamma, 0\}} C_\theta\sqrt{m d T} + \sigma C_\theta \sqrt{d \max\{d, R^2\}T}\right) \, .
    \end{equation*}
\end{theorem}

\begin{proof}[Proof of Theorem~\ref{thm:improved regret for linear model formal}]
Recall that, if we denote by $r_t$ the expected reward collected at round $t$ when Algorithm~\ref{alg:OM} plays $\xb_\tau$ in Eq.~\eqref{eq:OM action}, then the cumulative regret of Algorithm~\ref{alg:OM} can be decomposed as
\begin{align*}
    \regret_T
    &= \opt(T) - \sum_{t=1}^T r_t \\
    &= \opt(T) - \sum_{\tau=1}^{T/(m+L)}
    \left(
        \sum_{i=1}^m r(\xb_{\tau,i})
        + \sum_{i=m+1}^{m+L} r(\xb_{\tau,i})
    \right) \\
    &= \opt(T)
    - \sum_{\tau=1}^{T/(m+L)} \sum_{i=1}^m r(\xb_{\tau,i})
    - \sum_{\tau=1}^{T/(m+L)} \tilde{r}(\xb_\tau) \, , \numberthis \label{eq:linear regret eq}
\end{align*}
where the last equality follows from the definition of $\tilde{r}(\xb)$ in Eq.~\eqref{eq:proxy reward}.

By the boundedness of the linear reward, we have for all $\tau \in [T/(m+L)]$,
\begin{equation} \label{eq:linear regret reward upper bound}
    \bigl| r(\xb_{\tau,i}) \bigr| 
    = \bigl| \langle \xb_{\tau,i}, \Ab_{\tau,i-1} \thetab^* \rangle \bigr|
    \le \|\Ab_{\tau,i-1}\thetab^*\|_2 \|\xb_{\tau,i}\|_2
    \le \|\Ab_{\tau,i-1}\|_* \|\thetab^*\|_2
    \le (m+1)^{\gamma^+} C_\theta \, .
\end{equation}

Then, by Proposition~\ref{prop:proxy_upper}, $\opt(T)$ admits the upper bound
\begin{align} 
    \opt(T)
    &\le
    \frac{T}{m+L}\,\tilde{r}(\tilde{\xb})
    + \frac{m (m+1)^{\gamma^+} C_\theta}{m+L}\, T
    + L (m+1)^{\gamma^+} C_\theta \notag\\
    &\le
    \frac{T}{m+L}\,\tilde{r}(\tilde{\xb})
    + \frac{m R C_\theta}{m+L}\, T
    + L R C_\theta,
    \label{eq:linear regret opt bound}
\end{align}
where $\tilde{\xb}$ denotes the optimal block of the proxy reward $\tilde{r}(\xb)$ and $R:=(m+1)^{\gamma^+}$.  

Applying Eq.~\eqref{eq:linear regret opt bound} into Eq.~\eqref{eq:linear regret eq} we have
\begin{align*}
    \regret_T 
    & \le 
    \sum_{\tau=1}^{T/(m+L)} \tilde{r}(\tilde{\xb}) + \frac{m R C_\theta}{m+L}\, T
    + L R C_\theta 
    - \sum_{\tau=1}^{T/(m+L)} \sum_{i=1}^m r(\xb_{\tau,i})
    - \sum_{\tau=1}^{T/(m+L)} \tilde{r}(\xb_\tau) 
    \\
    & \le \sum_{\tau=1}^{T/(m+L)} \tilde{r}(\tilde{\xb}) - \sum_{\tau=1}^{T/(m+L)} \tilde{r}(\xb_\tau) 
    + \frac{2 m R C_\theta}{m+L}\, T
    + L R C_\theta \, 
    \\
    & \le \sum_{\tau=1}^{T/(m+L)} \{\tilde{r}(\tilde{\mathbf{x}}) - \tilde{r}(\mathbf{x}_\tau)\}
        + \frac{2 mR C_\theta}{m+L}\, T
    + LR C_\theta \, ,
    \\
&
=
\sum_{\tau=1}^{T/(m+L)} 
\sum_{i=m+1}^{m+L}
\{\langle 
\tilde{\bb}_i, \thetab^* 
\rangle
-
\langle 
\bb_{\tau,i}, \thetab^* 
\rangle	 \}
    + \frac{2 m R C_\theta}{m+L}\, T
    + LR C_\theta \, , \numberthis \label{eq:regretupper}
\end{align*}
where in the second inequality we use the upper bound of linear reward in Eq.~\eqref{eq:linear regret reward upper bound}, and the last equality directly follows from the definition of proxy regret.

For $\beta_\tau:=\sigma\sqrt{ 2\log T + d \log \left(1+ \frac{\tau(m+L)(m+1)^{2\max\{\gamma,0\}}}{d\lambda}\right)}+\sqrt{\lambda}C_\theta$, let $\Omega$ be the event that $\|\thetab^* -\hat{\thetab}_\tau\|_{V_{\tau}}\le \beta_\tau$ holds for all $\tau \in [\frac{T}{m+L}]$. Then, by Lemma~\ref{lem:linear confidence}, $\mathbb{P}(\Omega)\ge 1-{\frac{1}{T}}$. 
On the event $\Omega$, 
\begin{align}
&\sum_{\tau=1}^{T/(m+L)} 
\sum_{i=m+1}^{m+L}
\langle 
\tilde{\bb}_i, \thetab^* 
\rangle
-
\langle 
\bb_{\tau,i}, \thetab^* 
\rangle \notag	\\ 
&\quad \le \sum_{\tau=1}^{T/(m+L)} 
\left\langle \sum_{i=m+1}^{m+L}
\bb_{\tau,i}, \tilde{\thetab}_\tau 
\right\rangle
-
\left\langle \sum_{i=m+1}^{m+L}
\bb_{\tau,i}, \thetab^* 
\right\rangle \notag \qquad \big(\text{where} \big(\sum\bb_{\tau,i}, \tilde{\thetab_\tau}\big) \text{is optimistic}\big)	\\ 
&\quad \le \sum_{\tau=1}^{T/(m+L)} 
\left\langle \sum_{i=m+1}^{m+L}
\bb_{\tau,i}, \hat{\thetab}_\tau 
\right\rangle
+\left\langle \sum_{i=m+1}^{m+L}
\bb_{\tau,i}, \tilde{\thetab}_\tau -\hat{\thetab}_\tau 
\right\rangle
-
\left\langle \sum_{i=m+1}^{m+L}
\bb_{\tau,i}, \hat{\thetab}_\tau
\right\rangle 
+\left\langle \sum_{i=m+1}^{m+L}
\bb_{\tau,i}, \hat{\thetab}_\tau-\thetab^* 
\right\rangle \notag 	\\ 
& \quad
\le
\sum_{\tau=1}^{T/(m+L)} 
\left\lVert \tilde{\thetab}_\tau -\hat{\thetab}_\tau\right\rVert_{\Vb_\tau} 
\left\lVert \sum_{i=m+1}^{m+L} \bb_{\tau,i}\right\rVert_{\Vb_\tau^{-1}} 
+
\left\lVert \hat{\thetab}_\tau-\thetab^*\right\rVert_{\Vb_\tau} 
\left\lVert \sum_{i=m+1}^{m+L} \bb_{\tau,i} \right\rVert_{\Vb_\tau^{-1}} \notag
\\
& \quad
\le
\sum_{\tau=1}^{T/(m+L)} 
2\beta_{\tau}
\left\lVert \sum_{i=m+1}^{m+L} \bb_{\tau,i}\right\rVert_{\Vb_\tau^{-1}} \notag \\
& \quad
\le \sum_{\tau=1}^{T/(m+L)} 
\sum_{i=m+1}^{m+L}
2 \beta_{\tau}
\lVert \bb_{\tau,i} \rVert_{\Vb_\tau^{-1}} \notag
\\
& \quad
\le
2
\beta_{T/(m+L)}
\sum_{\tau=1}^{T/(m+L)} 
\sum_{i=m+1}^{m+L}
\lVert \bb_{\tau,i} \rVert_{\Vb_\tau^{-1}}  \notag
\\
& \quad
=2
\beta_{T/(m+L)} \left[
\sum_{\tau=1}^{T/(m+L)} 
\sum_{i=m+1}^{m+L}
\left(
\lVert \bb_{\tau,i} \rVert_{\Vb_\tau^{-1}}
- 
\lVert \bb_{\tau,i} \rVert_{\Vb_{\tau, i}^{-1}}
\right)
+
\sum_{\tau=1}^{T/(m+L)} 
\sum_{i=m+1}^{m+L}
\lVert \bb_{\tau,i} \rVert_{\Vb_{\tau, i}^{-1}} \right]
\label{eq:decomposition}
\end{align}

For the first term $\sum_{\tau=1}^{T/(m+L)} 
\sum_{i=m+1}^{m+L}
\left(
\lVert \bb_{\tau,i} \rVert_{\Vb_\tau^{-1}}
- 
\lVert \bb_{\tau,i} \rVert_{\Vb_{\tau, i}^{-1}}
\right)$ of Eq.~\eqref{eq:decomposition},  let $\lambda_{\tau,1}, \dots, \lambda_{\tau,d}$ be the eigenvalues of $\Vb_\tau$ for all $\tau\in \left[ \frac{T}{m+L} \right] $. Since $\Vb_{\tau+1} \succeq \Vb_{\tau, i}$ for all $i \in [m+1, m+L]$, we have $\lVert \bb_{\tau, i}\rVert_{\Vb_{\tau,i}^{-1}} \ge \lVert \bb_{\tau, i}\rVert_{\Vb_{\tau+1}^{-1}} $, and hence, 

\begin{align}
    \sum_{\tau=1}^{T/(m+L)} 
    \sum_{i=m+1}^{m+L}
    \left(
    \lVert \bb_{\tau,i} \rVert_{\Vb_\tau^{-1}}
    - 
    \lVert \bb_{\tau,i} \rVert_{\Vb_{\tau, i}^{-1}}
    \right)\notag
    &
    \le
    \sum_{\tau=1}^{T/(m+L)} 
    \sum_{i=m+1}^{m+L}
    \left(
    \lVert  \bb_{\tau,i} \rVert_{\Vb_\tau^{-1}}
    - 
    \lVert \bb_{\tau,i} \rVert_{\Vb_{\tau+1}^{-1}}
    \right) \notag
    \\
    &
    =
    \sum_{\tau=1}^{T/(m+L)} 
    \sum_{i=m+1}^{m+L}
    \lVert \bb_{\tau,i} 
    \rVert
    \cdot
    \left(
    \lVert \tilde{\bb}_{\tau,i} \rVert_{\Vb_\tau^{-1}}
    - 
    \lVert \tilde{\bb}_{\tau,i} \rVert_{\Vb_{\tau+1}^{-1}}
    \right)
    \tag{$\because \tilde{\bb}_{\tau,i} := \bb_{\tau,i}/\lVert \bb_{\tau,i} \rVert$}
    \\
    &
    \le
    R
    \sum_{\tau=1}^{T/(m+L)} 
    \sum_{i=m+1}^{m+L}
    \left(
    \lVert \tilde{\bb}_{\tau,i} \rVert_{\Vb_\tau^{-1}}
    - 
    \lVert \tilde{\bb}_{\tau,i} \rVert_{\Vb_{\tau+1}^{-1}}
    \right).
    \tag{$\because \lVert \bb_{\tau,i} \rVert_2\le R$}
\end{align}
By Lemma~\ref{lem:jin semi}, the right hand side is bounded by 
\begin{align}
    \sum_{\tau=1}^{T/(m+L)} 
    \sum_{i=m+1}^{m+L}
    \left(
    \lVert \bb_{\tau,i} \rVert_{\Vb_\tau^{-1}}
    - 
    \lVert \bb_{\tau,i} \rVert_{\Vb_{\tau, i}^{-1}}
    \right) 
    &
    \le
    R
    \sum_{\tau=1}^{T/(m+L)} 
    \sum_{i=m+1}^{m+L}
    \left(
    \sum_{j=1}^{d}
    \frac{2}{\sqrt{\lambda_{\tau,j}}}
    -
    \frac{2}{\sqrt{\lambda_{\tau+1,j}}}	
    \right) \notag
    \\
    &
    =
    LR
    \sum_{\tau=1}^{T/(m+L)} 
    \left(
    \sum_{j=1}^{d}
    \frac{2}{\sqrt{\lambda_{\tau,j}}}
    -
    \frac{2}{\sqrt{\lambda_{\tau+1,j}}}	
    \right)\notag
    \\
    &
    \le 
    LR 
    \sum_{j=1}^{d}
    \frac{2}{\sqrt{\lambda_{1,j}}} 
    \tag{$\because$ \text{telescoping argument}}
    \\
    & 
    =
    \frac{2LRd}{\sqrt{\lambda}} 
    \label{eq:difference}
\end{align}

Now we bound the second term $\sum_{\tau=1}^{T/(m+L)} 
\sum_{i=m+1}^{m+L}
\lVert \bb_{\tau,i} \rVert_{\Vb_{\tau, i}^{-1}}$ of Eq.~\eqref{eq:decomposition}. Since $\lVert \bb_{\tau,i} \rVert_2\le R$, setting $\lambda \ge R^2$ and applying the elliptical potential lemma (Lemma~\ref{lem:elliptical potential}) yields

\begin{align}
    \sum_{\tau=1}^{T/(m+L)} 
    \sum_{i=m+1}^{m+L}
    \lVert \bb_{\tau,i} \rVert_{\Vb_{\tau, i}^{-1}}
    \le 
    \sqrt{\frac{T}{m+L}\cdot L \cdot \sum_{\tau=1}^{T/(m+L)} 
    \sum_{i=m+1}^{m+L}
    \lVert \bb_{\tau,i} \rVert_{\Vb_{\tau, i}^{-1}}^2} 
    \le
    \sqrt{
    2
    Td
   \log\left( 1+\frac{R^2T}{d\lambda}\right)} \, ,
    \label{eq:ellip_onestep}
\end{align}
where the first inequality follows by Cauchy-Schwarz inequality.

Applying Eq.~\eqref{eq:ellip_onestep} and Eq.~\eqref{eq:difference} to Eq.~\eqref{eq:decomposition}, on the event $\Omega$, 
\begin{equation} \label{eq:regretupper1}
\sum_{\tau=1}^{T/(m+L)} 
\sum_{i=m+1}^{m+L}
\langle 
\tilde{\bb}_i, \thetab^* 
\rangle
-
\langle 
\bb_{\tau,i}, \thetab^* 
\rangle	 
\le 
2\beta_{T/(m+L)}\left(\sqrt{2Td
\log\left( 1+\frac{R^2T}{d\lambda}\right)
}
+\frac{2LRd}{\sqrt{\lambda}}
\right)
\end{equation}
Therefore, plugging Eq.~\ref{eq:regretupper1} into  Eq.~\eqref{eq:regretupper}, we have 
\begin{align*}
    \regret_T &\le
\sum_{\tau=1}^{T/(m+L)} 
\sum_{i=m+1}^{m+L}
\{\langle 
\tilde{\bb}_i, \thetab^* 
\rangle
-
\langle 
\bb_{\tau,i}, \thetab^* 
\rangle	 \}
    + \frac{2 m R C_\theta}{m+L}\, T
    + L R C_\theta  \\
&
\le
2 
\left( \sigma
\sqrt{
2\log T
+
d\log\left(1+ \frac{R^2T}{d\lambda}\right)
}
+\sqrt{\lambda}C_\theta
\right)
\times
\left(
\sqrt{
2
Td
\log\left( 1+\frac{R^2T}{d\lambda}\right)
}
+\frac{2LRd}{\sqrt{\lambda}}
\right)+2R C_\theta T  \mathbb{P}(\bar{\Omega})\\
&~~~+ \frac{2 m R C_\theta}{m+L}\, T
    + LRC_\theta 
\\
&
=
4 \left( \sigma
\sqrt{
	\frac{2\log T}{d}
	+
	\log\left( 1+\frac{R^2T}{d\lambda}\right)
}
+
\frac{\sqrt{\lambda}C_\theta}{\sqrt{d}}
\right)
\left(
\frac{LRd^{3/2}}{\sqrt{\lambda}}+d\sqrt{T\log\left( 1+\frac{R^2T}{d\lambda}\right)}
\right)
+ \frac{2 m R C_\theta}{m+L}\, T
    + (L+2)RC_\theta 
\end{align*}

By setting $L=\lceil m^{1/2}d^{-1/2}T^{1/2}\rceil -m$ and $\lambda=\mathcal{O}(\max\{d,R^2\})$, the regret has order of 
$$\tilde{\mathcal{O}}
(\sigma RC_\theta\sqrt{mdT} + \sigma C_\theta\sqrt{d\max\{d, R^2\}T}).$$
\end{proof}

\section{Proof of Theorem~\ref{thm:regret main algorithm}}\label{appx:proof for main algorithm}

\begin{theorem}[Regret bound of $\algname$ (formal)]
\label{thm:warmstarted-certified-regret-formal}
Suppose that Assumptions~\ref{assm:bounded domain}, \ref{assm:bounded link}, and~\ref{assm:self concordance} hold.
    If Algorithm~\ref{alg:main algorithm} is run with
    \(\beta=\Ocal(\nu R C_{\theta}\sqrt{d\nu R C_{\theta}^{2}+d\log(RT)})\),
    \(\lambda=\Ocal(d\nu^3R^3C_\theta)\), and
    \(L=\lfloor\sqrt{mT}\rfloor-m\), then Algorithm~\ref{alg:main algorithm} satisfies the following regret:
\begin{equation*}
\label{eq:app-main-cscb-bound}
    \regret_T
    =
    \widetilde\Ocal\!\left(
        (m+L)\tau_{\rm warm}
        +
        \sqrt{mT}
        +
        d\sqrt T
        +
        \sqrt\kappa\,d^2m^{1/4}T^{1/4}
        +
        \kappa d^2
    \right).
\end{equation*}
Consequently, whenever the stopping rule is reached within \(\tau_{\rm warm}=\widetilde\Ocal(1)\) cycles, it satisfies
\begin{equation}
\label{eq:app-main-cscb-bound-final}
    \regret_T
    =
    \widetilde\Ocal\!\left(
        \sqrt{mT}
        +
        d\sqrt T
        +
        \sqrt\kappa\,d^2m^{1/4}T^{1/4}
        +
        \kappa d^2
    \right).
\end{equation}
\end{theorem}
\label{appx:proof-warmstarted-certified-scb}

\begin{remark}[Number of warm-up rounds]
   It is natural that the warm-up block $\xb^{\rm warm}$ increases the eigenvalue of the gram matrix $\Vb_\tau$ by $\alpha \frac{L}{d}$, for some $\alpha>0$, per one block cycle. Then the stopping rule $\lambda_{\min} (\Vb_\tau)\geq L/d$ is satisfied after at most $\lceil 1/\alpha\rceil$ warm-up cycles. In particular, if $\alpha=\Omega(1)$, then $\tau_{\rm warm}=O(1)$. Then, the regret imposed by the warm-up stage has the order of $\Ocal((m+L)\tau_{\rm warm})=\widetilde{\Ocal}(\sqrt{mT})$.      
\end{remark}
\begin{proof}[Proof of Theorem~\ref{thm:regret main algorithm}]
Recall that, if we denote by $r_t = \EE[y_t \mid \Fcal_{t-1}, \xb_t]$ the expected reward collected at round $t$ when Algorithm~\ref{alg:main algorithm} plays $\xb_\tau$, then the cumulative regret of Algorithm~\ref{alg:main algorithm} can be decomposed as
\begin{align*}
    \regret_T
    &= \opt(T) - \sum_{t=1}^T r_t \\
    &= \opt(T) - \sum_{\tau=1}^{T/(m+L)}
    \left(
        \sum_{i=1}^m r(\xb_{\tau,i})
        + \sum_{i=m+1}^{m+L} r(\xb_{\tau,i})
    \right) \\
    &= \opt(T)
    - \sum_{\tau=1}^{T/(m+L)} \sum_{i=1}^m r(\xb_{\tau,i})
    - \sum_{\tau=1}^{T/(m+L)} \tilde{r}(\xb_\tau) \, , \numberthis \label{eq:glm regret eq}
\end{align*}
where the last equality follows from the definition of $\tilde{r}(\xb)$ in Eq.~\eqref{eq:proxy reward}.

By the boundedness of the reward for generalized linear model, we have for all $\tau \in [T/(m+L)]$,
\begin{align} \label{eq:glm regret reward upper bound}
    \bigl| r(\xb_{\tau,i}) \bigr| 
    &= \bigl| \mu(\langle \xb_{\tau,i}, \Ab_{\tau,i-1} \thetab^* \rangle) \bigr| \notag\\
    &\le \max \{|\mu (\|\Ab_{\tau,i-1}\thetab^*\|_2 \|\xb_{\tau,i}\|_2)|,~~ |\mu (-\|\Ab_{\tau,i-1}\thetab^*\|_2 \|\xb_{\tau,i}\|_2)|\} \notag,\\
    &\le \max \{|\mu(\|\Ab_{\tau,i-1}\|_* \|\thetab^*\|_2)|, ~~\|\mu(-\|\Ab_{\tau,i-1}\|_* \|\thetab^*\|_2)|\} \notag \\
    &\le \max\{|\mu((m+1)^{\gamma^+} C_\theta)|, ~~|\mu(-(m+1)^{\gamma^+} C_\theta)| \}=:R_\mu \, .
\end{align}
We denote this uniform upper bound on the reward by $R_\mu := \max\{|\mu((m+1)^{\gamma^+} C_\theta)|, ~~|\mu(-(m+1)^{\gamma^+} C_\theta)| \}$.
Then, by Proposition~\ref{prop:proxy_upper}, $\opt(T)$ admits the upper bound
\begin{align} 
    \opt(T)
    &\le
    \frac{T}{m+L}\,\tilde{r}(\tilde{\xb})
    + \frac{m R_\mu}{m+L}\, T
    + LR_\mu
    \label{eq:glm regret opt bound}
\end{align}
where $\tilde{\xb}$ denotes the optimal block of the proxy reward $\tilde{r}(\xb)$.  

Applying Eq.~\eqref{eq:glm regret opt bound} into Eq.~\eqref{eq:glm regret eq} we have
\begin{align*}
    \regret_T 
    & \le 
    \sum_{\tau=1}^{T/(m+L)} \tilde{r}(\tilde{\xb}) + \frac{m R_\mu}{m+L}\, T
    + L R_\mu
    - \sum_{\tau=1}^{T/(m+L)} \sum_{i=1}^m r(\xb_{\tau,i})
    - \sum_{\tau=1}^{T/(m+L)} \tilde{r}(\xb_\tau) 
    \\
    & \le \sum_{\tau=1}^{T/(m+L)} \tilde{r}(\tilde{\xb}) - \sum_{\tau=1}^{T/(m+L)} \tilde{r}(\xb_\tau) 
    + \frac{2 m R_\mu}{m+L}\, T
    + L R_\mu\, 
    \\
    & \le \sum_{\tau=1}^{T/(m+L)} \{\tilde{r}(\tilde{\mathbf{x}}) - \tilde{r}(\mathbf{x}_\tau)\}
        + \frac{2 mR_\mu}{m+L}\, T
    + LR_\mu \, ,
    \\
    &\le R_\mu\tau_{\rm warm}(m+L) + \sum_{\tau=\tau_{\rm warm}+1}^{T/(m+L)}
\sum_{i=m+1}^{m+L}
\{\tilde{r}(\tilde{\mathbf{x}}) - \tilde{r}(\mathbf{x}_\tau)\}
        + \frac{2 mR_\mu}{m+L}\, T
    + LR_\mu \, , \numberthis \label{eq:glm regretupper}
\end{align*}
where in the second inequality we use the upper bound of glm reward in Eq.~\eqref{eq:glm regret reward upper bound}, and the last equality directly follows from the maximum regret during initial exploration. 

By Lemma~\ref{lem:zhang_confi}, if we define $\Omega$ be the event that $\|\thetab^* -{\thetab}_{\tau,i}\|_{\Hb_{\tau,i}}\le \beta$ holds for all $\tau \in [{T \over m+L}]$ and $i \in [m+1, m+L+1]$, where $\beta:=\mathcal{O}\left(\nu R C_{\thetab} \sqrt{d\big(\nu R C_{\thetab}^2 + \ln{R T}\big)}\right)$, we have that $\mathbb{P}(\Omega)\ge 1-{1 \over T}$. On \(\Omega\), Lemma~\ref{lem:app-selected-block-decomposition} implies
\[
    \sum_{\tau={\tau_{\rm warm}+1}}^{T/(m+L)}
    \left[\tilde r(\tilde\xb)-\tilde r(\xb_\tau)\right]
    \le
    2\sum_{\tau={\tau_{\rm warm}+1}}^{T/(m+L)}\sum_{i=m+1}^{m+L}
   \overline\gamma_{\tau,i}(\xb_\tau)
    +
    \sum_{\tau={\tau_{\rm warm}+1}}^{T/(m+L)}\xi_\tau(\xb_\tau).
\]
Applying Lemma~\ref{lem:app-cumulative-shrunken-widths} and Lemma~\ref{lem:app-cumulative-excess} gives
\[
    \sum_{\tau={\tau_{\rm warm}+1}}^{T/(m+L)}
    \left[\tilde r(\tilde\xb)-\tilde r(\xb_\tau)\right]
    =
    \widetilde\Ocal\!\left(
        d\sqrt T
        +
        \sqrt\kappa\,d^2\sqrt L
        +
        \kappa d^2
    \right).
\]
Therefore, substituting this bound into Eq.~\eqref{eq:glm regretupper}, we have 
\begin{align*}
    \regret_T
    &=  (m+L)\tau_{\rm warm}+ \widetilde\Ocal\!\left(
       d\sqrt T
        +
        \sqrt\kappa\,d^2\sqrt L
        +
        \kappa d^2
    \right) +2R_\mu T  \mathbb{P}(\bar{\Omega}) + \frac{2 mR_\mu}{m+L}\, T
    + LR_\mu \\
    &=\widetilde\Ocal\!\left((m+L)\tau_{\rm warm}+ 
       d\sqrt T
        +
        \sqrt\kappa\,d^2\sqrt L
        +
        \kappa d^2 + \frac{2 mR_\mu}{m+L}\, T
    + LR_\mu 
    \right),
\end{align*}
since $\mathbb{P}(\bar{\Omega}) \le {1 \over T}$. If \(L=\lfloor\sqrt{mT}\rfloor-m\) and \(\tau_{\rm warm}=\widetilde\Ocal(1)\), then the regret is bounded by
\[
        \regret_T
    =
    \widetilde\Ocal\!\left(
        \sqrt{mT}
        +
        d\sqrt T
        +
        \sqrt\kappa\,d^2m^{1/4}T^{1/4}
        +
        \kappa d^2
    \right).
\]
\end{proof}

\subsection{Technical lemmas for Theorem~\ref{thm:regret main algorithm}}
\begin{lemma}[Estimation error decomposition]
\label{lem:app-selected-block-decomposition} Suppose that $\|\thetab^* -{\thetab}_{\tau,i}\|_{\Hb_{\tau,i}}\le \beta$ holds for all $\tau \in [{T \over m+L}]$ and $i \in [m+1, m+L+1]$. For every $\tau\ge \tau_{\rm warm}+1$ and candidate block $\xb\in\Xcal^{m+L}$,
\begin{equation*}
    |\tilde r(\xb)-\hat r_\tau(\xb)|
    \le
    \Gamma_\tau^{\mathrm{safe}}(\xb)
    \le
    \sum_{i=m+1}^{m+L}\gamma_{\tau,i}^{\mathrm{pre}}({\mathbf{x}}_{i}).
\end{equation*}
Moreover, if $\tilde\xb\in\argmax_{\xb\in\Xcal^{m+L}}\tilde r(\xb)$ is a proxy-optimal block, then
\begin{equation*}
    \tilde r(\tilde\xb)-\tilde r(\xb_\tau)
    \le
    2\sum_{i=m+1}^{m+L}
   \overline\gamma_{\tau,i}(\xb_\tau)+\xi_\tau(\xb_\tau),
\end{equation*}
where \(\xi_\tau(\xb):=[\Gamma_\tau^{\mathrm{safe}}(\xb)-\Gamma_\tau^{\mathrm{SCB}}(\xb)]_+\). 
\end{lemma}

\begin{proof}[Proof of Lemma~\ref{lem:app-selected-block-decomposition}]
Fix $\tau$ and $\xb_\tau$. For brevity, write $\bb_i=\bb_i(\xb)$, $z_i=\bb_i^\top\thetab_\tau$, and $h_i=\bb_i^\top(\thetab^*-\thetab_\tau)$ for $i=m+1,\ldots,m+L$. By Taylor's theorem with integral remainder,
\begin{align}
    \tilde r(\xb)-\hat r_\tau(\xb)
    &=
    \sum_{i=m+1}^{m+L}
    \dot\mu(\bb_i^\top\thetab_\tau)\bb_i^\top(\thetab^*-\thetab_\tau)
    +
    \sum_{i=m+1}^{m+L}
    h_i^2\int_0^1(1-s)\ddot\mu(\bb_i^\top\thetab_\tau+s h_i)\,ds .
    \label{eq:app-block-cert-taylor}
\end{align}
Then, the first term in Eq.~\eqref{eq:app-block-cert-taylor} is $\gb_\tau(\xb)^\top(\thetab^*-\thetab_\tau)$ , so the confidence event gives
\begin{equation}
\label{eq:app-block-cert-first}
    |\gb_\tau(\xb)^\top(\thetab^*-\thetab_\tau)|
    \le
    \|\gb_\tau(\xb)\|_{\Hb_\tau^{-1}}
    \|\thetab^*-\thetab_\tau\|_{\Hb_\tau}
    \le
    \beta\|\gb_\tau(\xb)\|_{\Hb_\tau^{-1}} .
\end{equation}
For the second term, Assumptions~\ref{assm:bounded link} and~\ref{assm:self concordance} imply $|\ddot\mu(z)|\le \nu U_\mu$ on the feasible score range. Hence
\begin{align}
    \left|
    \sum_{i=m+1}^{m+L}
    h_i^2\int_0^1(1-s)\ddot\mu(\bb_i^\top\thetab_\tau+s h_i)\,ds
    \right|
    &\le
    {\nu U_\mu\over 2}
    \sum_{i=m+1}^{m+L}h_i^2 
    \notag\\
    &=
    {\nu U_\mu\over 2}
    (\thetab^*-\thetab_\tau)^\top\Qb(\xb)(\thetab^*-\thetab_\tau)
    \notag\\
    &\le
    {\nu U_\mu\over 2}\beta^2
    \left\|\Hb_\tau^{-1/2}\Qb(\xb)\Hb_\tau^{-1/2}\right\|_* .
    \label{eq:app-block-cert-second}
\end{align}
Applying Eq.~\eqref{eq:app-block-cert-first} and Eq.~\eqref{eq:app-block-cert-second} to Eq.~\eqref{eq:app-block-cert-taylor} proves $|\tilde r(\xb)-\hat r_\tau(\xb)|
    \le
    \Gamma_\tau^{\mathrm{safe}}(\xb)$.

From the triangle inequality, we have that
\begin{equation}
\label{eq:app-block-dom-first}
    \left\|
    \sum_{i=m+1}^{m+L}
    \dot\mu(\bb_i^\top\thetab_\tau)\bb_i
    \right\|_{\Hb_\tau^{-1}}
    \le
    \sum_{i=m+1}^{m+L}
    \dot\mu(\bb_i^\top\thetab_\tau)
    \|\bb_i\|_{\Hb_\tau^{-1}} .
\end{equation}
Also, since $\Hb_\tau^{-1/2}\Qb(\xb)\Hb_\tau^{-1/2}$ is positive semidefinite,
\begin{align}
    \left\|
    \Hb_\tau^{-1/2}\Qb(\xb)\Hb_\tau^{-1/2}
    \right\|_* \le
    \operatorname{tr}\left(\Hb_\tau^{-1}\Qb(\xb)\right)
    =
    \sum_{i=m+1}^{m+L}
    \|\bb_i\|_{\Hb_\tau^{-1}}^2 .
    \label{eq:app-block-dom-second}
\end{align}
Eq.~\eqref{eq:app-block-dom-first} and Eq.~\eqref{eq:app-block-dom-second} show $\Gamma_\tau^{\mathrm{safe}}(\xb)\le\sum_{i=m+1}^{m+L}\gamma_{\tau,i}^{\mathrm{pre}}({\mathbf{x}}_{i})$.

We now prove the second part of the lemma. The prediction certificate implies $\tilde r(\xb)\le \hat r_\tau(\xb)+\Gamma_\tau^{\mathrm{safe}}(\xb)\le \hat r_\tau({\xb})
   +
   \max\left\{
        \Gamma_\tau^{\mathrm{SCB}}({\xb}),
        \Gamma_\tau^{\mathrm{safe}}({\xb})
   \right\}$ for every $\xb$. Thus, by the optimality of $\xb_\tau$,
\begin{equation}
\label{eq:app-optimality-cscb}
    \tilde r(\tilde\xb)
    \le
    \hat r_\tau(\tilde{\xb})
   +
   \max\left\{
        \Gamma_\tau^{\mathrm{SCB}}(\tilde{\xb}),
        \Gamma_\tau^{\mathrm{safe}}(\tilde{\xb})
   \right\}
    \le \hat r_\tau(\xb_\tau)
   +
   \max\left\{
        \Gamma_\tau^{\mathrm{SCB}}(\xb_\tau),
        \Gamma_\tau^{\mathrm{safe}}(\xb_\tau)
   \right\}.
\end{equation}
The same certificate also gives $\tilde r(\xb_\tau)\ge \hat r_\tau(\xb_\tau)-\Gamma_\tau^{\mathrm{safe}}(\xb_\tau)$. Combining this lower bound with Eq.~\eqref{eq:app-optimality-cscb} yields
\begin{align}
    \tilde r(\tilde\xb)-\tilde r(\xb_\tau)
    &\le
    \Gamma_\tau^{\rm SCB}(\xb_\tau)+\xi_\tau(\xb_\tau)+\Gamma_\tau^{\mathrm{safe}}(\xb_\tau)
    \notag\\
    &=
    -\sum_{i=m+1}^{m+L}\gamma_{\tau,i}^{\mathrm{pre}}({\mathbf{x}}_{\tau})
    +2\sum_{i=m+1}^{m+L}
   \overline\gamma_{\tau,i}(\xb_\tau)
    +\xi_\tau(\xb_\tau)
    +\Gamma_\tau^{\mathrm{safe}}(\xb_\tau)
    \notag\\
    &\le
    2\sum_{i=m+1}^{m+L}
   \overline\gamma_{\tau,i}(\xb_\tau)+\xi_\tau(\xb_\tau),
    \label{eq:app-selected-decomp-proof}
\end{align}
where the last inequality uses $\Gamma_\tau^{\mathrm{safe}}(\xb_\tau)\le\sum_{i=m+1}^{m+L}\gamma_{\tau,i}^{\mathrm{pre}}({\mathbf{x}}_{\tau})$.
\end{proof}

The following lemma shows that the cumulative shrunken widths $ \overline\gamma_{\tau,i}(\xb_\tau)$ is bounded by $\widetilde{\Ocal}(d\sqrt{T})$. 

\begin{lemma}[Cumulative shrunken widths]
\label{lem:app-cumulative-shrunken-widths}
On \(\Omega\), for the certified-SCB cycles selected by Algorithm~\ref{alg:main algorithm},
\[
    \sum_{\tau={\tau_{\rm warm}+1}}^{T/(m+L)}\sum_{i=m+1}^{m+L}
   \overline\gamma_{\tau,i}(\xb_\tau)
    =
    \widetilde\Ocal\!\left(
        d\sqrt T
        +
        \sqrt{\kappa L}\,d^2
        +
        \kappa d^2
    \right).
\]
\end{lemma}

\begin{proof}[Proof of Lemma~\ref{lem:app-cumulative-shrunken-widths}]
Recall that $\sum_{\tau={\tau_{\rm warm}+1}}^{T/(m+L)}\bar\Gamma_\tau(\xb_\tau)\le\sum_{\tau=1}^{T/(m+L)}\sum_{i=m+1}^{m+L}\overline\gamma_{\tau,i}(\xb_\tau)$, where 
\begin{equation*}
\overline\gamma_{\tau,i}(\xb_\tau) := \beta \dot{\mu}({\mathbf{b}}_{\tau,i}^\top \thetab_{\tau,i}^{+,\xb_\tau}) \|{\mathbf{b}}_{\tau,i}\|_{\minH_{\tau,i}(\xb_\tau)^{-1}} + \frac{\nu U_\mu}{2}{\beta^2}\|{\mathbf{b}}_{\tau,i}\|_{\minH_{\tau,i}(\xb_\tau)^{-1}}^2.
\end{equation*}

By the definition of $\thetab_{\tau,i}^+$ and $\thetab_{\tau,i}^{-,\xb_\tau}$, by the mean valued theorem, we have that 
\begin{align*}
\dot{\mu}(\mathbf{b}_{\tau,i}^\top \thetab_{\tau,i}^{+,\xb_\tau}) - \dot{\mu}(\mathbf{b}_{\tau,i}^\top \thetab_{\tau,i}^{-,\xb_\tau}) 
\le \nu U_\mu|\mathbf{b}_{\tau,i}^\top \thetab_{\tau,i}^{+,\xb_\tau} - \mathbf{b}_{\tau,i}^\top \thetab_{\tau,i}^{-,\xb_\tau}| \le \nu U_\mu \|\thetab_{\tau,i}^{+,\xb_\tau}-\thetab_{\tau,i}^{-,\xb_\tau}\|_{\mathbf{H}_\tau}\|\mathbf{b}_{\tau,i}\|_{\mathbf{H}_\tau^{-1}} \le {4 \beta \nu  U_\mu} \|\mathbf{b}_{\tau,i}\|_{\mathbf{H}_\tau^{-1}} ~\ .~~\quad
\end{align*}

Then, we have that 
\begin{align}
\sum_{\tau=1}^{T/(m+L)}&\sum_{i=m+1}^{m+L}\overline\gamma_{\tau,i}(\xb_\tau) \notag \\ & = \sum_{\tau=1}^{T/(m+L)}\sum_{i=m+1}^{m+L}\beta \dot{\mu}({\mathbf{b}}_{\tau,i}^\top \thetab_{\tau,i}^{+,\xb_\tau}) \|{\mathbf{b}}_{\tau,i}\|_{\minH_{\tau,i}(\xb_\tau)^{-1}} + \frac{\nu U_\mu}{2}{\beta^2}\|{\mathbf{{\mathbf{b}}_{\tau,i}}}\|_{\minH_{\tau,i}(\xb_\tau)^{-1}}^2 \notag \\
&  \le \beta\sum_{\tau=1}^{T/(m+L)}\sum_{i=m+1}^{m+L} (\dot{\mu}({\mathbf{b}}_{\tau,i}^\top \thetab_{\tau,i}^{-,\xb_\tau}) + 4\beta\nu U_\mu \|{\mathbf{b}}_{\tau,i}\|_{\mathbf{H}_\tau^{-1}}) \|{\mathbf{b}}_{\tau,i}\|_{\minH_{\tau,i}(\xb_\tau)^{-1}} + \frac{\nu U_\mu}{2}{\beta^2 }\sum_{\tau=1}^{T/(m+L)}\sum_{i=m+1}^{m+L}\|{\mathbf{b}}_{\tau,i}\|_{\minH_{\tau,i}(\xb_\tau)^{-1}}^2 \notag \\
& \le \beta\sum_{\tau=1}^{T/(m+L)}\sum_{i=m+1}^{m+L} \dot{\mu}({\mathbf{b}}_{\tau,i}^\top \thetab_{\tau,i}^{-,\xb_\tau}) \|{\mathbf{b}}_{\tau,i}\|_{\minH_{\tau,i}(\xb_\tau)^{-1}}
+ \frac{9\nu U_\mu}{2}{\beta^2}\sum_{\tau=1}^{T/(m+L)}\sum_{i=m+1}^{m+L}\|{\mathbf{b}}_{\tau,i}\|_{\mathrm{H}_\tau^{-1}}\|{{{\mathbf{b}}_{\tau,i}}}\|_{\minH_{\tau,i}(\xb_\tau)^{-1}} \label{eq:glm regret last term bound}
\end{align}
Define 
\begin{align*}
    \mathbf{L}_{\tau,i} &:= \sum_{\tau'=1}^{\tau-1}\sum_{i'=m+1}^{m+L}\dot{\mu}(\mathbf{b}_{\tau',i'}^\top \thetab_{\tau',i'}^{-,\xb_{\tau'}}) \mathbf{b}_{\tau',i'}\mathbf{b}_{\tau',i'}^\top + \sum_{i'=m+1}^{i-1}\dot{\mu}(\mathbf{b}_{\tau,i'}^\top \thetab_{\tau,i'}^{-,\xb_{\tau}}) \mathbf{b}_{\tau,i'}\mathbf{b}_{\tau,i'}^\top+\lambda \Ib_d\\
    \mathbf{W}_{\tau,i} &:= \sum_{\tau'=1}^{\tau-1}\sum_{i'=m+1}^{m+L}\mathbf{b}_{\tau',i'}\mathbf{b}_{\tau',i'}^\top + \sum_{i'=m+1}^{i-1}\mathbf{b}_{\tau,i'}\mathbf{b}_{\tau,i'}^\top+\kappa \lambda \mathbf{I}_d\\
    \mathbf{W}_\tau &:= \sum_{\tau'=1}^{\tau-1}\sum_{i=m+1}^{m+L}\mathbf{b}_{\tau',i}\mathbf{b}_{\tau',i}^\top + \kappa \lambda \mathbf{I}_d ~\ . 
\end{align*}

Then, by the definition of $\mintheta$ and $\kappa$, we have that $\mathbf{L}_{\tau,i} \preceq \minH_{\tau,i}(\xb_\tau)$, $\Hb_{\tau} \succeq {1 \over \kappa} \Wb_\tau$, and $\minH_{\tau,i}(\xb_\tau) \succeq {1 \over \kappa}\Wb_{\tau,i}$. Then, applying these facts to Eq. (\ref{eq:glm regret last term bound}), it follows that 

\begin{align}
\sum_{\tau=1}^{T/(m+L)}&\sum_{i=m+1}^{m+L}\overline\gamma_{\tau,i}(\xb_\tau) \notag  \\
& \le \beta\sum_{\tau=1}^{T/(m+L)}\sum_{i=m+1}^{m+L} \dot{\mu}({\mathbf{b}}_{\tau,i}^\top \thetab_{\tau,i}^{-,\xb_\tau}) \|{\mathbf{b}}_{\tau,i}\|_{\minH_{\tau,i}(\xb_\tau)^{-1}}
+\frac{9\nu U_\mu}{2}{\beta^2}\sum_{\tau=1}^{T/(m+L)}\sum_{i=m+1}^{m+L}\|{\mathbf{b}}_{\tau,i}\|_{\Hb_\tau^{-1}}\|{{{\mathbf{b}}_{\tau,i}}}\|_{\minH_{\tau,i}(\xb_\tau)^{-1}} \notag \\
& \le \beta\sum_{\tau=1}^{T/(m+L)}\sum_{i=m+1}^{m+L} \dot{\mu}({\mathbf{b}}_{\tau,i}^\top \thetab_{\tau,i}^{-,\xb_\tau}) \|{\mathbf{b}}_{\tau,i}\|_{{\Lb}_{\tau,i}^{-1}}
+\frac{9\nu U_\mu \kappa}{2} {\beta^2}\sum_{\tau=1}^{T/(m+L)}\sum_{i=m+1}^{m+L}\|{\mathbf{b}}_{\tau,i}\|_{\mathbf{W}_\tau^{-1}}\|{{\mathbf{b}}_{\tau,i}}\|_{{\mathbf{W}}_{\tau,i}^{-1}} ~\ . \quad \label{eq:first_second_decompotion}
\end{align}

For the first order term of Eq.(\ref{eq:first_second_decompotion}), by Lemma~\ref{lem:elliptical potential}, 
\begin{align}
\beta&\sum_{\tau=1}^{T/(m+L)}\sum_{i=m+1}^{m+L} \dot{\mu}({\mathbf{b}}_{\tau,i}^\top \thetab_{\tau,i}^{-,\xb_\tau}) \|{\mathbf{b}}_{\tau,i}\|_{{\Lb}_{\tau,i}^{-1}} \notag \\&=  \beta \sqrt{U_\mu}\sum_{\tau=1}^{T/(m+L)}\sum_{i=m+1}^{m+L}  \| \sqrt{\dot{\mu}({\mathbf{b}}_{\tau,i}^\top \thetab_{\tau,i}^{-,\xb_\tau})}{\mathbf{b}}_{\tau,i}\|_{{\Lb}_{\tau,i}^{-1}} \notag\\
    &=\Ocal\left(\nu R C_\theta \sqrt{d\nu R C_\theta^2U_\mu+dU_\mu\log (RT)}) \right) \times \Ocal \left(\sqrt{dT\log(\lambda+ U_\mu R^2T/d)}\right)
    \notag\\
    &= \Ocal\left(d\nu R C_\theta \sqrt{\nu R C_\theta^2 U_\mu T\log(\lambda+ U_\mu R^2T/d)+U_\mu\log (RT)\log(\lambda+ U_\mu R^2T/d)}\right) ~~\quad  \label{eq:first_order}
\end{align}

For the second order term of Eq.(\ref{eq:first_second_decompotion}), we have that
\begin{align}
\sum_{\tau=1}^{T/(m+L)}&\sum_{i=m+1}^{m+L}\|{\mathbf{b}}_{\tau,i}\|_{\Wb_\tau^{-1}}\|{\mathbf{{\mathbf{b}}_{\tau,i}}}\|_{{\Wb}_{\tau,i}^{-1}} \notag \\
&=
\sum_{\tau=1}^{T/(m+L)} 
\sum_{i=m+1}^{m+L}
\left(
\|{\mathbf{b}}_{\tau,i}\|_{\Wb_\tau^{-1}}\|{\mathbf{{\mathbf{b}}_{\tau,i}}}\|_{{\Wb}_{\tau,i}^{-1}}
- 
\lVert 
{b}_{\tau,i} \rVert_{\Wb_{\tau, i}^{-1}}^2
\right)
+
\sum_{\tau=1}^{T/(m+L)} 
\sum_{i=m+1}^{m+L}
\lVert 
{b}_{\tau,i} \rVert_{\Wb_{\tau, i}^{-1}}^2 \notag\\
&\le
\sum_{\tau=1}^{T/(m+L)} 
\sum_{i=m+1}^{m+L} \|{{\mathbf{b}}_{\tau,i}}\|_{{\Wb}_{\tau,i}^{-1}}
\left(
\lVert 
{b}_{\tau,i} \rVert_{\Wb_\tau^{-1}}
- 
\lVert 
{b}_{\tau,i} \rVert_{\Wb_{\tau, i}^{-1}}
\right) +\Ocal(d \log (\kappa \lambda + R^2 T/d))\quad.\quad (\because \text{Lemma~\ref{lem:elliptical potential}}) \notag\\
\end{align}
Let $\lambda_{\tau,1}, \ldots, \lambda_{\tau,d}$ be the eigenvalues of $\Wb_\tau$. Then, by lemma~\ref{lem:jin semi}, 
\begin{align}
\sum_{\tau=1}^{T/(m+L)}&\sum_{i=m+1}^{m+L}\|{\mathbf{b}}_{\tau,i}\|_{\Wb_\tau^{-1}}\|{{\mathbf{b}}_{\tau,i}}\|_{{\Wb}_{\tau,i}^{-1}} \notag \\
&\le
 \sum_{\tau=1}^{T/(m+L)} 
\sum_{i=m+1}^{m+L} \|{{\mathbf{b}}_{\tau,i}}\|_{{\Wb}_{\tau,i}^{-1}}
\left(
\sum_{j=1}^{d}
\frac{2R}{\sqrt{\lambda_{\tau,j}}}
-
\frac{2R}{\sqrt{\lambda_{\tau+1,j}}}	
\right)  +\Ocal(d \log (\kappa \lambda + R^2 T/d)) \quad\quad
\notag\\
&\le
\sqrt{ \sum_{\tau=1}^{T/(m+L)} 
\sum_{i=m+1}^{m+L} \|{{\mathbf{b}}_{\tau,i}}\|_{{\Wb}_{\tau,i}^{-1}}^2 }  \sqrt{\sum_{\tau=1}^{T/(m+L)} 
\sum_{i=m+1}^{m+L}
\left(
\sum_{j=1}^{d}
\frac{2R}{\sqrt{\lambda_{\tau,j}}}
-
\frac{2R}{\sqrt{\lambda_{\tau+1,j}}}	
\right)^2 }+\Ocal(d \log (\kappa \lambda + R^2 T/d)) \notag \quad \quad \\
&
\le
\Ocal(\sqrt{d} \log^{1/2} (\kappa \lambda + R^2 T/d)) \sqrt{\sum_{\tau=1}^{T/(m+L)} 
\sum_{i=m+1}^{m+L}
\left(
\sum_{j=1}^{d}
\frac{2R}{\sqrt{\lambda_{\tau,j}}}\right)^2 
-
\left(\sum_{j=1}^{d}\frac{2R}{\sqrt{\lambda_{\tau+1,j}}}	
\right)^2 }\notag\\
& \quad \quad +\Ocal(d \log (\kappa \lambda + R^2 T/d)) \qquad \qquad \qquad \qquad \qquad \qquad \qquad \qquad \qquad \qquad \qquad  (\because \text{Lemma~\ref{lem:elliptical potential}})  \notag\\
&\le
\Ocal(\sqrt{d} \log^{1/2} (\kappa \lambda + R^2 T/d)) \sqrt{L
\left(
\sum_{j=1}^{d}
\frac{2R}{\sqrt{\lambda_{1,j}}}\right)^2 
 } +\Ocal(d \log (\kappa \lambda + R^2 T/d))\qquad (\because \text{telescoping argument}) \notag\\
& 
=
\Ocal(\sqrt{d} \log^{1/2} (\kappa \lambda + R^2 T/d))   \times \Ocal\left(\sqrt{L}{dR \over \sqrt{\kappa \lambda}}\right) +\Ocal(d \log (\kappa \lambda + R^2 T/d))\notag\\
&=
\Ocal\left(\sqrt{L}{d^{3/2}R\log^{1/2} (\kappa \lambda + R^2 T/d))  \over \sqrt{\kappa \lambda}}+d \log (\kappa \lambda + R^2 T/d)\right) \notag \\
& =
\Ocal\left(\sqrt{L}{d\log^{1/2} (\kappa \lambda + R^2 T/d))  \over \sqrt{\kappa \nu^3 R C_\theta}}+d \log (\kappa \lambda + R^2 T/d)\right),
\quad \label{eq:second_order}
\end{align}
where $\lambda=O(d\nu^3 R^3C_\theta)$, and hence we can bound the second order term of Eq.(\ref{eq:first_second_decompotion}) by 
\begin{align}
    &\frac{9\nu U_\mu \kappa}{2}{\beta^2}\sum_{\tau=1}^{T/(m+L)}\sum_{i=m+1}^{m+L}\|{\mathbf{b}}_{\tau,i}\|_{\mathbf{W}_\tau^{-1}}\|{{\mathbf{b}}_{\tau,i}}\|_{{\mathbf{W}}_{\tau,i}^{-1}} \notag \\
    &~~~=\Ocal(\nu U_\mu \kappa) \times \Ocal (d\nu^3 R^3 C_\theta^4  + d \nu^2 R^2 C_\theta^2 \log (RT)) \times \Ocal\left(\sqrt{L}{d\log^{1/2} (\kappa \lambda + R^2 T/d))  \over \sqrt{\kappa \nu^3 R C_\theta}}+d \log (\kappa \lambda + R^2 T/d)\right) \notag\\
    &~~~=\Ocal (d\nu^3 R^3 C_\theta^4  + d \nu^2 R^2 C_\theta^2 \log (RT)) \times \Ocal\left(\sqrt{L\kappa}{dU_\mu\log^{1/2} (\kappa \lambda + R^2 T/d))  \over \sqrt{ \nu R C_\theta}}+\kappa  d \nu U_\mu \log (\kappa \lambda + R^2 T/d)\right)\notag\\ 
    &~~~=\Ocal(d\nu^3 R^3 C_\theta^4 \log (RT))\times \Ocal\left(\sqrt{L\kappa}{dU_\mu\log^{1/2} (\kappa \lambda + R^2 T/d)  \over \sqrt{ \nu R C_\theta}}+ \kappa d \nu U_\mu  \log (\kappa \lambda + R^2 T/d)\right) \notag\\
    &~~~=\Ocal\left(\sqrt{L\kappa}{d^2 \nu^{5/2} R^{5/2} C_\theta^{7/2} U_\mu\log^{1/2} (\kappa \lambda + R^2 T/d)\log (RT) }+\kappa d^2 \nu^4 C_\theta^4 U_\mu  \log (\kappa \lambda + R^2 T/d)\log (RT)\right) \label{eq:second_order}
\end{align}

Plugging Eq.(\ref{eq:first_order}), Eq.(\ref{eq:second_order}) into Eq.(\ref{eq:first_second_decompotion}), we have that
\begin{align*}
\sum_{\tau=1}^{T/(m+L)}&\sum_{i=m+1}^{m+L}\overline\gamma_{\tau,i}(\xb_\tau)\notag  \\
& \le \beta\sum_{\tau=1}^{T/(m+L)}\sum_{i=m+1}^{m+L} \dot{\mu}({\mathbf{b}}_{\tau,i}^\top \thetab_{\tau,i}^{-,\xb_\tau}) \|{\mathbf{b}}_{\tau,i}\|_{{\Lb}_{\tau,i}^{-1}}
+ \frac{9\nu U_\mu \kappa}{2}{\beta^2} \sum_{\tau=1}^{T/(m+L)}\sum_{i=m+1}^{m+L}\|{\mathbf{b}}_{\tau,i}\|_{\mathbf{W}_\tau^{-1}}\|{{\mathbf{b}}_{\tau,i}}\|_{{\mathbf{W}}_{\tau,i}^{-1}}\\
& \le  \Ocal\left(d\nu R C_\theta \sqrt{\nu R C_\theta^2 U_\mu T\log(\lambda+ U_\mu R^2T/d)+U_\mu \log (RT)\log(\lambda+ U_\mu R^2T/d)}\right) \notag\\
& ~~~ + \Ocal\left(\sqrt{L\kappa}{d^2 \nu^{5/2} R^{5/2} C_\theta^{7/2} U_\mu\log^{1/2} (\kappa \lambda + R^2 T/d)\log (RT) }+ \kappa d^2 \nu^4 C_\theta^4 U_\mu \log (\kappa \lambda + R^2 T/d)\log (RT)\right) \notag \\
&=\widetilde{\Ocal}(d\sqrt{T}+\sqrt{\kappa L}d^2 + \kappa d^2 ).
\end{align*}
\end{proof}

The following lemma shows that the cumulative excess certification $\xi_\tau$ can not be the leading term with respect to $T$.
\begin{lemma}[Cumulative excess certification]
\label{lem:app-cumulative-excess}
Suppose that $\|\thetab^* -{\thetab}_{\tau,i}\|_{\Hb_{\tau,i}}\le \beta$ holds for all $\tau \in [{T \over m+L}]$ and $i \in [m+1, m+L+1]$, and $\beta:=\mathcal{O}\left(\nu R C_{\thetab} \sqrt{d\big(\nu R C_{\thetab}^2 + \ln{R T}\big)}\right)$. For the block selected by Algorithm~\ref{alg:main algorithm},
\begin{equation*}
    \sum_{\tau=\tau_{\rm warm}+1}^{T/(m+L)}\xi_\tau(\xb_\tau)
    =
    \widetilde\Ocal\!\left(
        \sqrt {\kappa L}d^2+
        \kappa d^2
    \right).
\end{equation*}
\end{lemma}

\begin{proof}[Proof of Lemma~\ref{lem:app-cumulative-excess}]
Fix a certified cycle $\tau\ge{\tau_{\rm warm}+1}$ and write $\bb_{\tau,i}:=\bb_i(\xb_\tau)$. Define $h_{\tau,i}:=\|\bb_{\tau,i}\|_{\Hb_\tau^{-1}}$ and $\bar h_{\tau,i}:=\|\bb_{\tau,i}\|_{(\bar\Hb_{\tau,i}^{\xb_\tau})^{-1}}$. Since $\underline{\Hb}_{\tau,i}({\xb_\tau})\succeq\Hb_\tau$, we have $h_{\tau,i}\ge \bar h_{\tau,i}$. Also, by the definition of $\thetab_{\tau,i}^{+,\xb_\tau}$, $\dot\mu(\bb_{\tau,i}^\top\thetab_{\tau,i}^{+,\xb_\tau})\ge \dot\mu(\bb_{\tau,i}^\top\thetab_\tau)$. Therefore,
\begin{align}
        \gamma_{\tau,i}^{\mathrm{pre}}({\mathbf{x}}_{\tau})
    -
   \overline\gamma_{\tau,i}(\xb_\tau)
    &\le
    \beta U_\mu(h_{\tau,i}-\bar h_{\tau,i})
    +
    {\nu U_\mu\over 2}\beta^2(h_{\tau,i}^2-\bar h_{\tau,i}^2).
    \label{eq:app-single-excess-with-constants}
\end{align}
By the domination part of Lemma~\ref{lem:app-selected-block-decomposition} and the definition of $\xi_\tau$,
\begin{align}
\label{eq:app-excess-S1-S2}
    \sum_{\tau={\tau_{\rm warm}+1}}^{T/(m+L)}\xi_\tau(\xb_\tau)
    &\le 
    \sum_{\tau={\tau_{\rm warm}+1}}^{T/(m+L)}\left[
    \Gamma_\tau^{\mathrm{safe}}(\xb_\tau)
    +
    \sum_{i=m+1}^{m+L}\gamma_{\tau,i}^{\mathrm{pre}}({\mathbf{x}}_{\tau})
    -
    2\sum_{i=m+1}^{m+L}
   \overline\gamma_{\tau,i}(\xb_\tau)
    \right]_+ \notag\\
    &\le
    2\sum_{\tau={\tau_{\rm warm}+1}}^{T/(m+L)}\sum_{i=m+1}^{m+L}
    \left[
    \gamma_{\tau,i}^{\mathrm{pre}}({\mathbf{x}}_{\tau})
    -
   \overline\gamma_{\tau,i}(\xb_\tau)
    \right]_+ \notag\\
    & \le 2\beta U_\mu \sum_{\tau={\tau_{\rm warm}+1}}^{T/(m+L)}\sum_{i=m+1}^{m+L}(h_{\tau,i}-\bar h_{\tau,i})+\nu U_\mu\beta^2\sum_{\tau={\tau_{\rm warm}+1}}^{T/(m+L)}\sum_{i=m+1}^{m+L}(h_{\tau,i}^2-\bar h_{\tau,i}^2).
\end{align}

 On $\Omega$, since $\Hb_\tau\preceq\Hb_{\tau,i+1}$ and
$\|\thetab_{\tau,i+1}-\thetab_\tau\|_{\Hb_\tau}
    \le
    \|\thetab_{\tau,i+1}-\thetab^*\|_{\Hb_\tau}
    +
    \|\thetab^*-\thetab_\tau\|_{\Hb_\tau}
    \le 2\beta$, we have $\thetab_{\tau,i+1}\in\Ccal_\tau^{(2\beta)}$. Hence, by the definition of $\thetab_{\tau,i}^{-,\xb_\tau}$, $\dot\mu(\bb_{\tau,i}^\top\thetab_{\tau,i}^{-,\xb_\tau})\le \dot\mu(\bb_{\tau,i}^\top\thetab_{\tau,i+1})$. Thus the conservative full-block matrix satisfies $\underline{\Hb}_{\tau,m+L+1}({\xb_\tau})
    \preceq
    \Hb_{\tau+1}$. Since $\underline{\Hb}_{\tau,i}({\xb_\tau})\preceq\underline{\Hb}_{\tau,m+L+1}({\xb_\tau})$ for all $i \in [m+1, m+L]$, we have

\begin{align*}
    \sum_{\tau={\tau_{\rm warm}+1}}^{T/(m+L)}\sum_{i=m+1}^{m+L}(h_{\tau,i}-\bar h_{\tau,i})
    &\le\sum_{\tau={\tau_{\rm warm}+1}}^{T/(m+L)}\sum_{i=m+1}^{m+L}\|\bb_{\tau,i}\|_{\Hb_\tau^{-1}}
    -\|\bb_{\tau,i}\|_{(\underline{\Hb}_{\tau,m+L+1}({\xb_\tau}))^{-1}} \\
    & \le \sum_{\tau={\tau_{\rm warm}+1}}^{T/(m+L)}\sum_{i=m+1}^{m+L} \|\bb_{\tau,i}\|_{\Hb_\tau^{-1}}
    -\|\bb_{\tau,i}\|_{\Hb_{\tau+1}^{-1}} \\
    & \le 
    \sum_{\tau={\tau_{\rm warm}+1}}^{T/(m+L)}\sum_{i=m+1}^{m+L}R (\| {\bb_{\tau,i}/\|\bb_{\tau,i}\|_2}\|_{\Hb_\tau^{-1}}
    -\|{\bb_{\tau,i}/\|\bb_{\tau,i}\|_2}\|_{\Hb_{\tau+1}^{-1}}) \\
    & \le \sum_{\tau={\tau_{\rm warm}+1}}^{T/(m+L)}\sum_{i=m+1}^{m+L}R \sum_{q=1}^{d}
    \left(
        {2\over\sqrt{\lambda_q(\Hb_\tau)}}
        -
        {2\over\sqrt{\lambda_q(\underline{\Hb}_{\tau,m+L+1}({\xb_\tau}))}}
    \right) \\    
    & \le 2RL\sum_{q=1}^{d}{1\over\sqrt{\lambda_q(\Hb_{{\tau_{\rm warm}+1}})}}, \numberthis \label{eq:S1}
\end{align*}
where $\lambda_1(\Hb),\ldots, \lambda_d(\Hb)$ are the $d$ eigenvalues of $\Hb$ for any gram matrix $\Hb$. 

For the second term, the same comparison gives
\begin{align}
    \sum_{\tau={\tau_{\rm warm}+1}}^{T/(m+L)}\sum_{i=m+1}^{m+L} (h_{\tau,i}^2-\bar h_{\tau,i}^2)
    &\le
    \sum_{\tau={\tau_{\rm warm}+1}}^{T/(m+L)}\sum_{i=m+1}^{m+L}\bb_{\tau,i}^\top
    \left(\Hb_\tau^{-1}-(\underline{\Hb}_{\tau,m+L+1}({\xb_\tau})\right)
    \bb_{\tau,i}
    \notag\\
    &\le
    LR^2\sum_{\tau={\tau_{\rm warm}+1}}^{T/(m+L)}\sum_{i=m+1}^{m+L}\operatorname{tr}\left(\Hb_\tau^{-1}-(\underline{\Hb}_{\tau,m+L+1}({\xb_\tau}))^{-1}\right)
    \notag\\
    &\le
    LR^2\sum_{\tau={\tau_{\rm warm}+1}}^{T/(m+L)}\sum_{i=m+1}^{m+L}\operatorname{tr}\left(\Hb_\tau^{-1}-\Hb_{\tau+1}^{-1}\right) \notag\\
    &\le
    LR^2\operatorname{tr}(\Hb_{{\tau_{\rm warm}+1}}^{-1})
    =
    LR^2\sum_{q=1}^{d}{1\over\lambda_q(\Hb_{{\tau_{\rm warm}+1}})}. \label{eq:S2}
\end{align}
Applying Eq.~\eqref{eq:S1} and Eq.~\eqref{eq:S2} to Eq.~\eqref{eq:app-excess-S1-S2},
\begin{equation}
\label{eq:app-excess-before-stop}
    \sum_{\tau={\tau_{\rm warm}+1}}^{T/(m+L)}\xi_\tau(\xb_\tau)
    \le
    4\beta U_\mu RL
    \sum_{q=1}^{d}{1\over\sqrt{\lambda_q(\Hb_{{\tau_{\rm warm}+1}})}}
    +
    \nu U_\mu\beta^2R^2L
    \sum_{q=1}^{d}{1\over\lambda_q(\Hb_{{\tau_{\rm warm}+1}})}.
\end{equation}
By the $H$-warm-start stopping rule, $\lambda_q(\Hb_{{\tau_{\rm warm}+1}})\ge \lambda_q(\Vb_{{\tau_{\rm warm}+1}})/\kappa \ge L/\kappa d$ for all $q$. Therefore,
\begin{equation}
\label{eq:app-excess-after-stop}
    \sum_{\tau={\tau_{\rm warm}+1}}^{T/(m+L)}\xi_\tau(\xb_\tau)
    \le
    {4\beta U_\mu RLd\over\sqrt{L/\kappa d}}
    +
    {\nu U_\mu\beta^2R^2Ld\over L/\kappa d}.
\end{equation}
Since, by Lemma~\ref{lem:zhang_confi}, $\beta=O\big(\nu RC_\theta\sqrt{d(\nu RC_\theta^2+\log(RT))}\big)$ for $\lambda \ge 14d\eta\nu^2R^2$,
\begin{align}
    \sum_{\tau={\tau_{\rm warm}+1}}^{T/(m+L)}\xi_\tau(\xb_\tau)
    \le& O\big(\nu RC_\theta\sqrt{d(\nu RC_\theta^2+\log(RT))}\big)\times O(U_\mu R d^{3/2}\sqrt{\kappa L}) \\ &+ 
    O\big(\nu^3 R^3C_\theta^4d+\nu^2 R^2C_\theta^2d\log(RT)\big)\times O(\nu U_\mu R^2 \kappa d^{2}) \\
    & = \widetilde{\Ocal}(d^2\sqrt{\kappa L} +\kappa d^2). 
\end{align}
\end{proof}

\section{Approximation rate of proxy reward}
\label{appx:approximation}
\citet{clerici2024linear} propose a cyclic strategy for the linear bandit setting with memory, together with upper and lower bounds on the approximation error. We first present the following proposition, which establishes the corresponding upper bound. This result refines Proposition~2 of \citet{clerici2024linear} in two respects.

First, we correct a minor error in the proof of Proposition~2 of \citet{clerici2024linear}. They claim that ``there exists a block of length $L$ in the optimal sequence whose average expected reward is higher than $\mathrm{OPT}/T$.'' However, this claim is valid only when $T$ is divisible by $L$; otherwise, a simple counterexample exists. For instance, let $(r_t^*)_{t=1}^{3}=(2,1,2)$. Then the average reward of the optimal sequence is $5/3$, whereas the average reward of any block of length $2$ is $3/2$, which is strictly smaller than $5/3$. We address this issue by explicitly accounting for the rewards obtained in the remaining $T \bmod L$ rounds.

Second, \citet{clerici2024linear} express the approximation error as $\mathrm{OPT}-r(\tilde{\xb})$, namely, the difference between the cumulative optimal reward and the cumulative expected reward obtained by cyclically playing the proxy-optimal block $\tilde{\xb}$. In contrast, measuring the gap with respect to the cumulative proxy reward, $\mathrm{OPT}-\tilde{r}(\tilde{\xb})$, removes an unnecessary inequality in the regret analysis and leads to a sharper final bound. Therefore, we state the approximation-error bound in terms of the difference between the cumulative optimal reward and the cumulative proxy reward obtained by cyclically playing $\tilde{\xb}$.

Finally, as stated in our problem formulation, we assume without loss of generality that $T$ is a multiple of $m+L$. This assumption does not affect the regret order.

\begin{proposition}[Upper bound on approximation error, Proposition 2 of \citet{clerici2024linear}]
\label{prop:proxy_upper}
For any $L>0$, when playing  cyclically proxy optimal block $\tilde{\mathbf{x}}$, the cumulative proxy reward satisfies
\begin{equation*}
    \mathrm{OPT}
    -
    {T \over m+L}
    \tilde{r} (\tilde{\mathbf{x}})
    \;\le\;
    \frac{mR_\mu}{m+L}\, T+{L R_\mu}.
\end{equation*}
\end{proposition}
\begin{proof}[Proof of Proposition~\ref{prop:proxy_upper}]
    Let $(\xb_t^*)_{t\in[T]}$ be the true optimal sequence, $R_\mu$ be the maximum reward, and consider blocks of size $L$, i.e. $\zb_1:=(\xb_1,\ldots, \xb_L), \zb_2:=(\xb_2,\ldots, \xb_{L+1}),\ldots, \zb_{T-L+1}:=(\xb_T-L+1,\ldots, \xb_T)$. Since the sum of reward for the blocks $\zb_1, \zb_{L+1}, \ldots \zb_{L([{T \over L }]-1)+1}$ is $\text{OPT}-\sum_{t=L\cdot [T/L]+1}^T\mu(\xb_{t}^\top \Ab_{t-1}\thetab^*) \ge \text{OPT}-LR_\mu$, there is at least one block $\zb_{t_*}$ such that $\sum_{t=t_*}^{t_*+L-1}\mu (\xb_{t_*}^\top \Ab_{t_*-1} \thetab^*) \ge {\text{OPT}-LR_\mu \over [T /L]} \ge {L \over T}(\text{OPT}-LR_\mu )$. Consider the block $\xb_* =(\xb_{t_*-m},\ldots,\xb_{t_*}, \ldots, \xb_{t_*+L-1})$ of size $m+L$. Then, by the definition of proxy reward and proxy optimal block, $\tilde{r}(\tilde{\xb})\ge \tilde{r}(\xb_*) \ge {L \over T}(\text{OPT}-LR_\mu )$. Therefore,  
    \begin{align*}
        \mathrm{OPT}
    -
    {T \over m+L}
    \tilde{r} (\tilde{\mathbf{x}})
    \;&\le\mathrm{OPT}- {T \over m+L}\times{L \over T}(\text{OPT}-LR_\mu )\\
    & \le \mathrm{OPT}-{L \over m+L}\mathrm{OPT}+{L^2 R_\mu \over m+L}\\
    & \le {m \over m+L}\mathrm{OPT}+{L^2 R_\mu \over m+L}\\
    & \le {mR_\mu \over m+L}T+{L R_\mu }
    \end{align*}
\end{proof}

The following theorem shows that the above approximation bound is tight. 
\begin{proposition}[Lower bound on approximation error, Proposition 3 of \citet{clerici2024linear}]
\label{prop:proxy_lower}
Let $\mathrm{OPT}$ denote the optimal cumulative expected reward over horizon $T$. Then, there exists a problem instance and a proxy optimal block $\tilde{a}$ such that the cumulative proxy reward  when playing  cyclically $\tilde{\mathbf{a}}$ satisfies
\begin{equation*}
    \mathrm{OPT}
    -
     {T \over m+L}
    \tilde{r} (\tilde{\mathbf{a}})
    \;\ge\;
    \frac{mR_\mu}{m+L} T
\end{equation*}
\end{proposition}

\section{Auxiliary lemmas}
\label{ap_sec:lemmas}

    \begin{lemma}[Theorem 1 in~\citet{zhang2025generalized}]\label{lem:zhang_confi}
        Let $\delta \in (0,1]$. Set the step size to $\eta=1+\nu RC_{\thetab}$ and the regularization parameter to $\lambda =\max\{14d\eta \nu^2 R^2, 6 \eta \nu RC_{\thetab} U_\mu/g(\phi)\}$. 
        For all $\tau \in [T/(m+L]$ and $i \in [m+1, m+L]$, with probability at least $1-\delta$, we have $\thetab^* \in \Ccal_{\tau,i}(\delta)$:
        \begin{align*}
            \Ccal_{\tau,i}(\delta):=\left\{ \thetab \in \RR^d
            : \lVert \thetab_{\tau,i} - \thetab \rVert_{\Hb_{\tau,i}} \le \beta(\delta)   
            \right\},
        \end{align*}
        where $\beta(\delta)=\mathcal{O}\big(\nu RC_{\thetab} \sqrt{d( \nu RC_{\thetab}^2 + \ln\frac{R^2 T}{\delta}})\big)$.
    \end{lemma}

\begin{lemma}[Theorem 2 of \citet{abbasi2011improved}]
\label{lem:linear confidence}
Let $\hat{\thetab}_t$ be a ridge estimator of $\theta^*$ with regularization parameter $\lambda$. Assume that $\|\thetab^*\| \le C_\theta$, noise is conditionally $\sigma$-sub-gaussian, and the $l_2$ norm features are bounded by $R$. Then, for any $\delta>0$, with probability at least $1-\delta$, for all $0 \le t \le T$, $\theta^*$ lies in the set 
\begin{equation*}
    C_t(\delta) := \left\{ \thetab \in \RR^d : \|\hat{\thetab}_t- \thetab \|_{V_t} \le \sigma \sqrt{2 \log \left({1 \over \delta}\right)+ d \log \left({1 + {R^2 T \over \lambda d }}\right)}+\sqrt{\lambda}C_\theta\right\}. 
\end{equation*}
\end{lemma}

\begin{lemma}[Lemma 11 of \citet{abbasi2011improved}]
\label{lem:elliptical potential}
    Let $(\bb_t)_{t\ge 1}$ be a sequence in $\RR^d$ with $\|\bb_t\|_2 \le R$ for all $t\ge1$ and $V_t:=\lambda \Ib_d + \sum_{s=1}^{t-1}\bb_s \bb_s^\top$. If $\lambda \ge \max \{1, R^2 \}$, then 
    \begin{equation*}
        \sum_{t=1}^T \|\bb_t\|_{V_t^{-1}}^2 \le 2d \log \left(1 + {R^2 t \over \lambda d}\right).  
    \end{equation*}
\end{lemma}

\begin{lemma}[Lemma 3 of \citet{jin2021shrinking}]
\label{lem:jin semi}
    Let $\Ab \in \RR^{d \times d}$ be any symmetric positive definite matrix, and $\ub_1, \ub_2, \ldots, \ub_L \in \RR^d$ be any vectors. Let $\lambda_1, \lambda_2, \ldots, \lambda_d$ be eigenvalues of $\Ab$, and $\nu_1, \nu_2, \ldots, \nu_d$ be the eigenvalues of $\Ab + \sum_{k=1}^L\ub_k \ub_k^\top$. Then, for any $\xb \in \RR^d$ such that $\|\xb\|_2=1$, 
    \begin{equation*}
        \|\xb\|_{\Ab^{-1}}-\|\xb\|_{(\Ab + \sum_{k=1}^L\ub_k \ub_k^\top)^{-1}} \le \sum_{d'=1}^d {2 \over \sqrt{\lambda_{d'}}}-\sum_{d'=1}^d {2 \over \sqrt{\nu_{d'}}}.
    \end{equation*}
\end{lemma}

\section{Additional Experimental Details}
\label{appx:experiment-details}

We provide additional details for the experiments in Section~\ref{sec:experiments}.
For the experiments, we use a general memory matrix
\[
    \Ab_{t-1}
    =
    \left(
        \Ab_0+
        \sum_{s=1}^{m}
        \xb_{t-s}\xb_{t-s}^{\top}
    \right)^\gamma ,
\]
where \(\Ab_0\succeq 0\) is an initial memory matrix.
The action set consists of finite unit-norm feature vectors in \(\RR^d\), and the same action set and parameter \(\thetab^*\) are used for all algorithms within each environment.
For logistic rewards, rewards are sampled from a Bernoulli distribution with mean
\[
    \mu(\langle \xb_t,\Ab_{t-1}\thetab^*\rangle),
    \qquad
    \mu(z)=\frac{1}{1+\exp(-z)}.
\]
For linear rewards, the identity link is used.

\paragraph{Logistic rotting environment.}
We set \(\gamma=-5\), memory length \(m=5\), and feature dimension \(d=10\).
The action set consists of \(30\) arms sampled uniformly from the unit sphere.
The true parameter \(\thetab^\star\) is chosen from this set.

\paragraph{Logistic rising environment.}
Following the experimental setup of~\citet{clerici2024linear}, we use \(\gamma=1\), \(m=2\), and \(d=2\).
The action set is the unit ball \(\Bcal_d:= \{ \xb \in \RR^d : \| \xb \|_2 \le 1 \}\).
As mentioned by~\citet{clerici2024linear}, \texttt{Oracle Greedy} is optimal when \(\Ab_0=\Ib_d\).
To study a nontrivial rising setting, we follow the same idea as the previous work and use a non-isotropic initial matrix
\(\Ab_0=\eb_1\eb_1^\top\).
For each run, we sample \(\epsilon\sim\mathrm{Unif}[0.001,0.1]\) and set
$
    \thetab^*
    =
    [\sqrt{\epsilon},\sqrt{1-\epsilon}]^\top .
$

\paragraph{Linear rotting environment.}
We consider a linear rotting environment with \(\gamma=-5\), \(m=30\), and \(d=10\).
The action set is the unit ball \(\Bcal_d\), and \(\thetab^*\) is sampled from \(\Bcal_d\).

\paragraph{Computational resources.}
All experiments were run on a MacBook Air with an Apple M3 chip, \(8\) CPU cores, and \(16\)GB unified memory.
The operating system was macOS 15.1.
Each experiment was run locally on CPU.

\end{document}